\documentclass{article}
\usepackage[preprint]{neurips_2026}
\usepackage[utf8]{inputenc}
\usepackage[T1]{fontenc}
\usepackage{amsmath,amssymb,amsthm}
\usepackage{booktabs}
\usepackage{graphicx}
\usepackage{float}
\usepackage{hyperref}
\usepackage{url}
\usepackage{microtype}
\usepackage{xcolor}

\newtheorem{theorem}{Theorem}
\newtheorem{lemma}[theorem]{Lemma}
\newtheorem{corollary}[theorem]{Corollary}
\newtheorem{proposition}[theorem]{Proposition}
\newtheorem{remark}[theorem]{Remark}

\title{Transformers Can Learn Rules They've Never Seen:\\Proof of Computation Beyond Interpolation}

\author{
  Andy Gray \\
  Kortical
}

\begin{document}

\maketitle

\begin{abstract}
A central question in the debate over large language models is whether
transformers can learn rules they have never seen, or whether they can only
interpolate: predict new cases from their similarity to training examples. We
test this in a controlled setting where interpolation provably fails, so
success can only come from computation beyond interpolation. We train small
transformers to predict the rollout of a cellular automaton whose update rule
is pure XOR, and remove one entry of the rule's truth table from all direct
supervision. The missing entry's output is never shown to the model; its only
trace is indirect, as wrong values corrupt visible predictions at later
timesteps. Because XOR parity flips whenever one input bit is changed, every
one-bit neighbour of the missing entry carries the opposite label, and we
prove that similarity-based predictors, including nearest-neighbour, kernel,
and Gaussian-process methods, are forced to the wrong answer. A two-layer
transformer can nevertheless recover the missing entry, and circuit
extraction confirms it computes XOR exactly. Ablations show the recovery
depends on gradient signal propagating through multi-step prediction, and a
second, structurally unrelated benchmark on symbolic operator chains exhibits
the same capacity under ordinary autoregressive training. Together with a
constructive proof that a standard transformer block can implement exact
local Boolean rules, these results provide an existence proof that
transformers can learn rule structure not directly observed in training and
express it explicitly. This rules out the strongest architectural form of the
interpolation-only account, the claim that transformers cannot in principle
discover and communicate unseen rules, while leaving open when such behaviour
arises in large-scale language training.
\end{abstract}

\section{Introduction}

A persistent question in the debate over large language models is whether
transformers can learn rules they have never seen, or whether their apparent
generalisation reduces to interpolation: predicting new cases from their
similarity to training examples
\citep{bender2021dangers,chollet2019measure,belkin2021fit,mirzadeh2024gsm}.
The question is hard to settle at scale. Training corpora are vast and opaque,
so when a model produces an apparently novel result it is rarely possible to
rule out that something similar appeared somewhere in its data. We therefore
approach the question from the opposite direction: we construct a small, fully
controlled setting in which similarity-based interpolation provably fails, and
ask whether a standard transformer trained by gradient descent succeeds
anyway.

The construction is simple to state. We train a two-layer transformer to
predict the step-by-step evolution of a one-dimensional cellular automaton: a
row of binary cells, each updated by a fixed local rule of its neighbourhood.
The rule is pure XOR, so a cell's next value is the parity of its
neighbourhood. We then remove one entry of the rule's truth table from all
direct supervision, masking the training loss at every position and timestep
where that input pattern occurs; we call accuracy at these positions
\emph{holdout accuracy}, and accuracy elsewhere \emph{visible accuracy}. The
model never observes the missing entry's output. The entry still leaves a
trace, however: predicting the wrong value for it corrupts predictions at
later timesteps, at positions that are supervised.

XOR is what makes the test sharp. Parity flips whenever any single bit flips,
so every nearest neighbour of the missing pattern carries the opposite label.
Similarity-based interpolation is then not merely unreliable but provably
wrong: nearest-neighbour, kernel, Gaussian-process, and tree-ensemble
predictors are all forced to the incorrect output
(Results~\ref{thm:monotone}--\ref{lem:rf}). One could always hand-design a
representation in which the missing entry becomes interpolable, but
constructing that representation already requires knowing the answer. Our
claim is therefore scoped to similarity in the data as actually given,
including embeddings learned end-to-end, which at a fixed position preserve
the input's neighbourhood geometry (Proposition~\ref{prop:hamming},
Corollary~\ref{cor:embedding_gap}) and so cannot supply the answer by
similarity.

A two-layer transformer nevertheless recovers the missing entry, at up to
100\% holdout accuracy, and circuit extraction recovers an exact XOR
computation from its weights. Ablations locate the mechanism in the training
signal. We train the model to predict several consecutive timesteps and
compare three regimes: \emph{soft unrolling}, where the predicted
probabilities at each step are fed back as the input to the next, so gradients
flow through the entire rollout; \emph{hard unrolling}, where predictions are
first rounded to binary states and gradients pass through a straight-through
estimator; and \emph{no unrolling}, where the model predicts a single step.
Recovery of the missing entry is reliable under soft unrolling, rare under
hard unrolling, and absent without unrolling (Section~\ref{sec:results}), so
the indirect downstream signal is what carries the missing entry into the
model. Success also tracks how much visible evidence the data provides. Rules
with wider neighbourhoods, which leave more supervised patterns visible, are
learned faster and more reliably, and holdout accuracy stays at its baseline
until visible accuracy is high, then rises sharply, a phase transition in the
onset of indirect learning (Section~\ref{sec:results},
Appendix~\ref{app:ksweep}). A second
benchmark, chains of symbolic operators over integers with one operator pair
held out entirely, reproduces the phenomenon under ordinary autoregressive
training: the model infers the held-out composition, exceeds every
interpolation baseline, and writes the operators it has discovered as explicit
symbols in its output (Section~\ref{sec:symbolic}).

\textbf{Why minimal models.} We deliberately use small transformers and
synthetic data because this eliminates the confounds that make definitive
claims impossible at scale: data contamination, memorisation, and opaque
corpora. The question is not whether LLMs compute XOR; it is whether the
transformer architecture \emph{can} learn a rule entry absent from its direct
supervision. Minimal models answer this cleanly, and the mechanism we identify
requires only attention and feedforward layers, the same components present in
every deployed LLM. Appendix~\ref{app:explicit_local_boolean_block} makes the
capacity concrete with an explicit construction of any radius-1 Boolean rule
in a single standard transformer block, and
Appendix~\ref{app:compositional_local_rule_circuits} shows such primitives
compose across depth. Establishing the capacity in a setting where the
similarity-based alternative is provably excluded is stronger evidence than
suggestive results at scale, where that alternative can never be fully ruled
out.


\section{Related Work}

\textbf{Cellular automata learning.} Prior work shows that neural networks can learn CA dynamics, typically in-distribution or across rules \citep{gilpin2019cellular}. \citet{ca_constraint_reconstruction2020} reconstructed CA rules from nonconsecutive observations via constraint satisfaction, and \citet{elser_lal_2026_btf} generalized this into a projection-based Boolean threshold learning framework. Our setting instead asks whether a standard transformer trained by gradient descent can recover missing truth-table entries.

\textbf{Algorithmic generalisation.} \citet{anil2022exploring} and \citet{deletang2022neural} showed that transformers fail to generalise across sequence lengths, a finding widely cited as evidence of fundamental algorithmic limitations. We do not contest the specific finding---transformers do fail at length generalisation. However, the broad conclusion is falsified by our experiment: a standard transformer computes a linearly inseparable function entirely absent from its training data. Length generalisation failure reflects a length-specific limitation (position encodings, sequence structure), not a computational one.

\textbf{The interpolation debate.} \citet{bender2021dangers} argue that LLMs are ``stochastic parrots'' that recombine observed patterns. Our experiment provides a controlled counterexample to the strongest architectural form of that view: the target pattern is never observed during training, yet the transformer still recovers it at up to 100\%. We do not dispute that transformers often default to pattern matching; our result shows that the architecture is not limited to it.

\citet{chollet2019measure} distinguishes ``local generalisation'' from genuine computation; we provide the first clean empirical demonstration of this boundary. XOR's linear inseparability means no amount of local interpolation recovers the hidden pattern (0\% across all methods), yet the transformer crosses to genuine computation with a mathematical guarantee.

\citet{belkin2021fit} shows generalisation reduces to interpolation for smooth functions; XOR is discontinuous and linearly inseparable, so no smooth manifold connects training patterns to the hidden one. Kernel methods, the mathematical backbone of Belkin's framework, provably achieve 0\% (Theorems~\ref{thm:cm}--\ref{thm:gp}).

\citet{mirzadeh2024gsm} argue LLM reasoning is fragile pattern matching; our transformer computes a function it was never trained on, verified by circuit extraction. \citet{dziri2023faith} conclude transformers reduce compositional reasoning to ``linearised subgraph matching''; the hidden pattern has no subgraph to match against, and the extracted XOR polynomial with nonlinear interaction terms proves genuine computation.

\textbf{Attention as smoothing.} \citet{tsai2019transformer} interpret self-attention as a kernel smoother over value vectors. Our appendix result is complementary: while attention supplies smooth local routing, a standard attention+ReLU transformer block already implements exact local Boolean rules, so the full block is not merely a smoothing operator.

\textbf{Mechanistic interpretability.} We build on \citet{nanda2023progress}, who reverse-engineered Fourier circuits for modular addition, and \citet{power2022grokking,zhang2024initialization,zhou2023length}. Our work extends Nanda et al.\ in four ways: (i)~a linearly inseparable target function, eliminating the criticism that learned trig is indistinguishable from smooth interpolation; (ii)~indirect supervision---the model never sees hidden pattern outputs, inferring them through constraint propagation; (iii)~three-level representation tracing (input 0\% $\to$ embeddings 51\% $\to$ final layer 98\%), providing a more complete mechanistic story; and (iv)~a directly interpretable XOR polynomial (zero fit error, explicit interaction terms $LC$, $LR$, $CR$, $LCR$) rather than Fourier features. The result is a definitive proof of computation rather than a suggestive one.

\section{Experimental Setup}
\label{sec:setup}

\subsection{Cellular automata as testbed}

A one-dimensional CA is a binary grid of width $W$ evolving with periodic boundary conditions: at each timestep, every cell's next value is given by a fixed rule applied to the $2r+1$ cells centred on it, where $r$ is the neighbourhood \emph{radius}. For $r=1$ the neighbourhood is 3 bits, the left neighbour $L$, centre cell $C$, and right neighbour $R$, giving $2^3 = 8$ possible input patterns; $r=2$ uses nearest neighbours $L_1, R_1$ and second neighbours $L_2, R_2$. We test five rules, written with $\oplus$ (XOR), $\lor$ (OR), and $\land$ (AND): \textbf{Rule~150} (radius~1, pure 3-way XOR: $L \oplus C \oplus R$), \textbf{Rule~30} ($L \oplus (C \lor R)$), \textbf{Rule~106}, \textbf{Rule~D} (radius~2, $L_1 \oplus ((C \lor R_1) \land (L_2 \lor R_2))$), and \textbf{Rule~G} (radius~2, $(L_1 \oplus (C \lor R_1)) \lor (L_2 \land R_2)$).

\begin{figure}[htbp]
  \centering
  \includegraphics[width=0.7\textwidth]{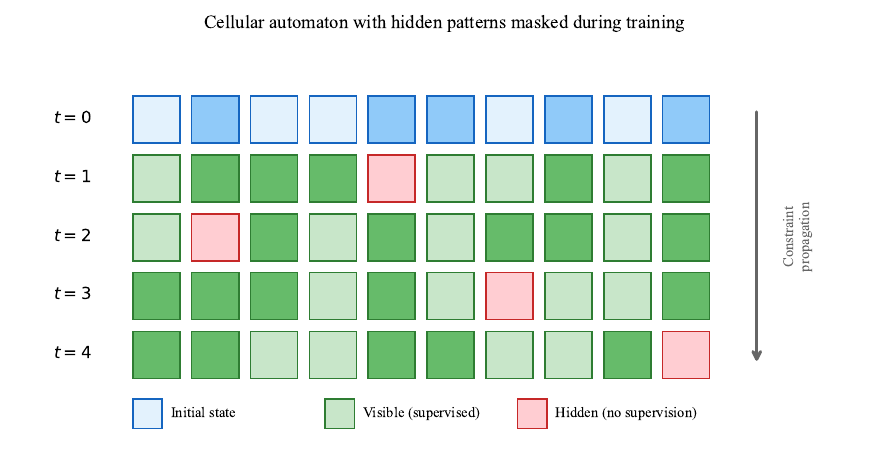}
  \caption{Experimental setup. A CA evolves from a random initial state ($t=0$). At each subsequent timestep, positions where hidden patterns occur (red) receive no supervision. Visible positions (green) provide training signal. Wrong hidden-pattern predictions at $t+1$ cascade into errors at visible positions at $t+2$, providing indirect gradient signal.}
  \label{fig:setup}
\end{figure}

\subsection{Hard-gap setup}

We select $k$ input patterns as ``hidden'' (Figure~\ref{fig:setup}). During training, outputs at positions where hidden patterns occur are masked at \emph{every} timestep; the model receives zero direct supervision on hidden pattern outputs. The model must infer hidden outputs through constraint propagation: wrong predictions at hidden positions cascade into errors at visible downstream positions ($t+2$ onward), generating indirect gradient signal.

\subsection{Why XOR provides a mathematical guarantee}

Pure XOR has the property that every input bit flip changes the output. For a held-out pattern, every nearest neighbour has the opposite label. We prove this defeats all standard interpolation methods.

Let $f: \{0,1\}^n \to \{0,1\}$ be pure XOR, with labels $y(x) := (-1)^{f(x)}$. Fix held-out $p$; training set $T := \{0,1\}^n \setminus \{p\}$.

\textbf{Parity-distance identity.} For any $q \in \{0,1\}^n$: $y(q) = y(p) \cdot (-1)^{d(q,p)}$, where $d(q,p)$ is the Hamming distance, the number of bit positions at which $q$ and $p$ differ. Odd-distance points have label $-y(p)$, even-distance points have label $y(p)$. All $\ell_m$ distances with finite $m$ on $\{0,1\}^n$ are monotone functions of Hamming distance.

\paragraph{Result 1: interpolation lower bounds.}
For held-out parity / Rule~150, all interpolation baselines tested here are forced away from the correct label: $k$-nearest neighbours (KNN; majority vote over the $k$ closest training points), similarity voting with any non-increasing distance weights, kernel predictors whose weights decay monotonically with Hamming distance, including the radial basis function (RBF) Gaussian kernel, Gaussian-process (GP) and kernel ridge regression (KRR), and axis-aligned decision trees along with Random Forests (RF) built from them. Appendix~\ref{app:proofs} gives the formal theorem statements and proofs. The guarantee is specific to pure XOR / Rule~150; Rule~D and Rule~G are included as structured comparison rules.

\subsection{Architecture and training}

The model is a standard post-LN transformer encoder: a scalar embedding ($1 \to 64$) plus learned absolute position embeddings, 2 layers, 4 heads, and a ReLU FFN of dimension 128, with a linear projection ($64 \to 1$) as output. Training uses the soft, hard, or no unrolling regimes defined in Section~1, unrolled over $s = 4$ consecutive timesteps, with gradients flowing through all unrolled steps. The loss is binary cross-entropy on visible positions only, with masks computed from ground-truth binary states. We use width 101, i.i.d.\ Bernoulli(1/2) initial rows, 20K training and 2K test samples, batch size 128, Adam with learning rate $10^{-3}$, and 10 seeds.

\section{Results}
\label{sec:results}

\subsection{Interpolation fails at every representation level}

\begin{figure}[htbp]
  \centering
  \includegraphics[width=0.7\textwidth]{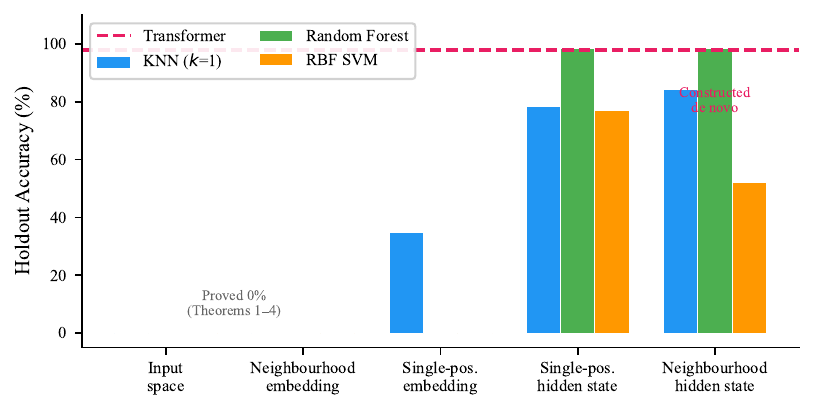}
  \caption{Interpolation accuracy across five representation levels for Rule~150. Input space: provably 0\% (Results~\ref{thm:monotone}--\ref{lem:rf}). Neighbourhood embeddings: empirically 0\%, with a fixed-position proof extension in Appendix~\ref{app:embedding_extension}. After transformer layers: 94--100\%. The XOR function is constructed de novo by computation.}
  \label{fig:interpolation}
\end{figure}

Table~\ref{tab:interpolation} and Figure~\ref{fig:interpolation} show interpolation accuracy at five levels, from raw input to post-transformer hidden states. In input space, all methods achieve 0\%, guaranteed by Results~\ref{thm:monotone}--\ref{lem:rf}. In pooled neighbourhood embeddings, all methods also achieve 0\% empirically. The Rule~150 interpolation-impossibility result extends to raw neighbourhood embeddings after conditioning on centre location (Appendix~\ref{app:embedding_extension}). After the transformer layers, Random Forests decode holdout patterns at 94--100\% from single-position hidden states. The function does not exist in the input or embedding; it is constructed by the transformer. High decodability from hidden states is a consequence of this computation, not an alternative explanation: the representation is absent pre-transformer and constructed de novo by the network.

\begin{table}[htbp]
\centering
\caption{Interpolation accuracy (\%) across five representation levels for Rule 150. Results summarize the six transformer runs achieving $\geq$95\% holdout accuracy.}
\label{tab:interpolation}
\small
\begin{tabular}{lccccc}
\toprule
Representation Level & KNN & GP & RBF SVM & RF & Transformer \\
\midrule
Input space (3-bit) & 0 & 0 & 0 & 0 & 96.8 \\
Nbhd.\ embedding (192-d) & 0 & 0 & 0 & 0 & 96.8 \\
Single-pos.\ embedding (64-d) & 34--42 & 5--27 & 0 & 0 & 96.8 \\
Single-pos.\ hidden (64-d) & 0--100 & 0--100 & 0--100 & 94--100 & 96.8 \\
Nbhd.\ hidden (192-d) & 72--100 & 0--100 & 0--100 & 95--100 & 96.8 \\
\bottomrule
\end{tabular}
\end{table}

A temporal baseline giving Random Forests the same inputs and supervision as the transformer achieves 100\% on supervised patterns but 0.2--0.4\% on hidden ones. XOR defeats interpolation regardless of temporal data.

\subsection{Main results}

\begin{table}[htbp]
\centering
\caption{Mean holdout accuracy across seeds (95\% CI). Outcomes are bimodal; the mean reflects the fraction of seeds that succeed.}
\label{tab:main}
\small
\begin{tabular}{lccr}
\toprule
Experiment & Mean & 95\% CI & Successes \\
\midrule
Rule D $k$=8 (soft) & 96.7\% & [94.1, 99.2] & 10/10 \\
Rule G $k$=8 (soft) & 99.5\% & [99.2, 99.7] & 10/10 \\
Rule D $k$=8 (hard/STE) & 65.5\% & [59.2, 71.9] & 1/10 \\
Rule 150 (soft, 220 ep.; initially failed patterns) & 71.5\% & [61.6, 81.4] & 47/60 \\
Rule 150 (soft, 50 ep.; all 8 patterns) & 10.6\% & [4.1, 17.1] & 8/80 \\
Rule D $k$=8, no unrolling (1 step) & 63.1\% & [62.7, 63.5] & 0/10 \\
Rule D $k$=8, leaky mask ($t$+1 only) & 99.3\% & [99.1, 99.6] & 10/10 \\
\bottomrule
\end{tabular}
\end{table}

Table~\ref{tab:main} shows the core results. For Rule~D $k$=8, soft unrolling dramatically outperforms hard/STE unrolling: 96.7\% vs 65.5\% ($p = 0.002$, Wilcoxon). Hard unrolling can nevertheless succeed alone: on Rule~30 it fully recovers held-out patterns on some seeds (Appendix~\ref{app:crossrule}, Table~\ref{tab:rule30_hard}). Multi-step structure alone is insufficient; the model needs differentiable signal propagation. Using the most common output from the hidden patterns as a constant, accuracy is 62.5\%, matching accuracy without unrolling (63.1\%). The leaky-mask control masks hidden positions only at $t+1$, so later direct supervision largely removes the hard-gap constraint.

\subsection{Constraint propagation is provably sufficient}

A GF(2) constraint solver confirms the multi-step rollout uniquely determines hidden outputs. With one timestep, identifiability is 0\%. With two timesteps, it is 100\% ($\sim$31 equations for $\sim$12.5 unknowns, full rank). The information exists; the question is whether the transformer can find it.

\subsection{Circuit analysis confirms XOR computation}

For Rule~150 models achieving 100\% holdout accuracy (2 models, independently trained), we fit the complete degree-3 polynomial over the 3-bit input. The interaction terms ($LC$, $LR$, $CR$, $LCR$) have large coefficients ($-25.3$ to $+46.1$), confirming the model computes a function with the nonlinear structure of XOR, not a linear or additive approximation. We note that multiple circuits can implement the same function \citep{mi_identifiability2025}; our claim is functional (the model computes XOR) rather than structural (this is the unique circuit).

Layer ablation: zeroing layer~0 drops accuracy by 79--87\%; zeroing layer~1 drops it by 100\%. The logit lens shows XOR is undecodable at the embedding and after layer~0 (0--3\%), then appears at 100\% after layer~1. Linear probes confirm the routing-plus-computation decomposition: neighbour bits become partially decodable after layer~0, while XOR jumps from chance to 98.4\% only after layer~1. Appendix~\ref{app:explicit_local_boolean_block} shows that a standard transformer block already contains an exact local-rule circuit with the same division of labour. Full circuit analysis in Appendix~\ref{app:circuits}.

\subsection{Phase transition and constraint density}

In pure XOR (Rule~150), where interpolation is provably impossible, holdout accuracy stays near 0\% until supervised accuracy exceeds $\sim$85\%, then rises sharply; full training dynamics are shown in Appendix~\ref{app:grokking}. In rules with internal structure that permits partial interpolation, holdout can rise earlier, but the same grokking-like transition to high holdout accuracy still appears around a similar threshold. Below this threshold, predictions are too noisy for wrong hidden outputs to produce detectable downstream errors.

Wider neighbourhoods help: Rule~D (radius~2, 32 total patterns, $k$=8 hidden) reaches 96.7\% (10/10 seeds), while Rule~150 (radius~1, 8 total patterns, $k$=1 hidden) reaches only 10.6\% in 50 epochs (8/80). The key factor is not the number of hidden patterns but the number of \emph{visible} ones: 24 visible constraints in radius-2 vs.\ 7 in radius-1. The $k$-sweep within a single rule confirms this: hiding \emph{more} patterns within Rule~D makes learning strictly harder, with a sharp cliff between 44\% and 50\% hidden (Appendix~\ref{app:ksweep}, Figure~\ref{fig:k_sweep}). Rule~G, whose OR-decomposable structure creates richer inter-pattern constraints, tolerates 75\% hidden at 97.1\% accuracy. A Conv1D with local receptive fields achieves 100\%, serving as a ceiling; the transformer reaches 96.7\% without built-in locality.

\subsection{Symbolic operator benchmark}
\label{sec:symbolic}

As an external-validity check on whether the capacity demonstrated above is specific to cellular automata, we train an encoder-decoder transformer (same depth: 2 layers, 4 heads, 64-dim, ${\sim}180$K parameters) on compositional chains of two binary operators over 6-bit integers. The input is five tokens $(a,b,c,d,u)$, where $(a\,\text{op}_1\,b)\,\text{op}_2\,c=d$ identifies the operator pair; the decoder emits $a\,\text{op}_1\,b=e;\ e\,\text{op}_2\,u=f$ plus the operator identities. Seven operators are used (XOR, OR, AND, NOR, NAND, LSHIFT, RSHIFT); one operator pair is held out entirely from training (48 of 49 pairs seen). We test all 49 holdout pairs (Appendix~\ref{app:symbolic}, Figure~\ref{fig:coverage}); four are analysed in detail below. We select the best checkpoint by peak eval-set performance and report its accuracy on a separate held-out test set of 500 examples per operator pair.

Table~\ref{tab:symbolic} shows holdout accuracy across four held-out pairs under three conditions. Three findings parallel the CA results. \textbf{(i)}~All holdout pairs exceed baseline accuracy (KNN and MLP score 0\%; KRR reaches 0--18.2\%); KRR uses the numeric inputs directly while the transformer discovers the operator identity and outputs it symbolically, a qualitatively different kind of generalisation. \textbf{(ii)}~Replacing operator symbols with arbitrary letters (``opaque'') leaves accuracy comparable on three pairs, and improves it on \texttt{\&L} ($+10.6$ points, consistent with seed fluctuation): arbitrary relabelling does not collapse the result. All embeddings are learned from scratch. \textbf{(iii)}~Removing intermediate computation steps (``label-only'') roughly halves accuracy, paralleling the CA finding that soft unrolling (96.7\%) far exceeds no unrolling (63.1\%): multi-step structure provides the constraint propagation channel. Further mechanistic details (logit lens, cross-attention maps, head ablation) are in Appendix~\ref{app:symbolic}.

\textbf{Causal substitution test.} To verify the model has learned the causal structure of the computation, we perform targeted input substitutions (Figure~\ref{fig:substitution}). Starting from 200 unambiguous \texttt{\^{}|} examples, we replace inputs $(c, d)$ with values consistent with a different second operator while holding $(a, b)$ fixed, or vice versa. The model correctly identifies the substituted operator 77--100\% of the time across all 7 target operators, while preserving the untouched operator (94--100\% when swapping $a,b$). This double dissociation confirms the model has learned which inputs determine which operator: not a surface correlation but the correct causal graph.

\begin{figure}[htbp]
  \centering
  \includegraphics[width=0.9\textwidth]{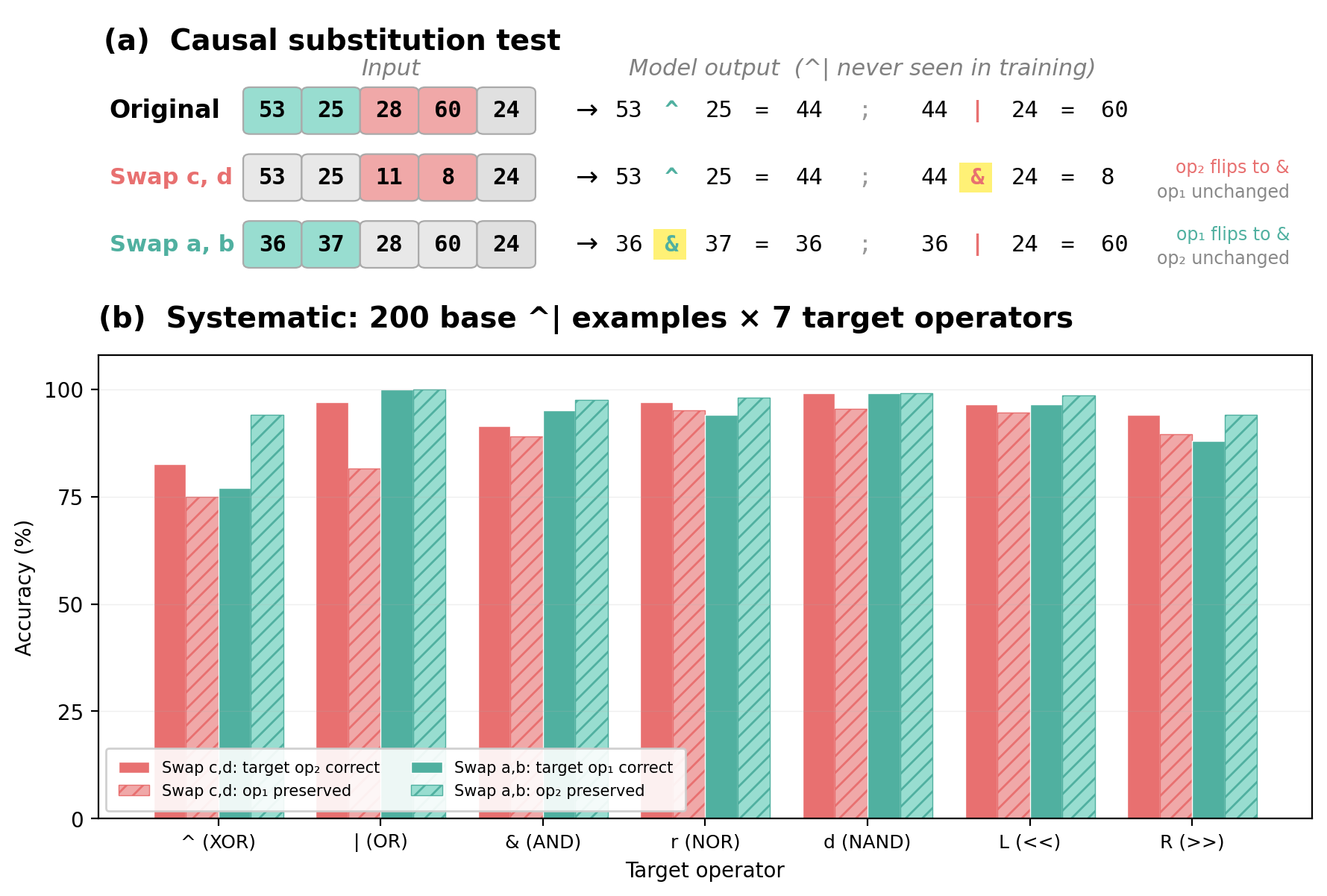}
  \vspace{-0.5em}
  \caption{Causal substitution test for the symbolic benchmark. (a)~Surgical input replacement: changing $(c,d)$ to values consistent with a different operator flips op\textsubscript{2} while preserving op\textsubscript{1}; changing $(a,b)$ flips op\textsubscript{1} while preserving op\textsubscript{2}. (b)~Systematic results over 200 base examples per target operator confirm a clean double dissociation.}
  \label{fig:substitution}
\end{figure}

\begin{table}[htbp]
\centering
\caption{Symbolic benchmark: test-set holdout accuracy (\%) averaged over 3 seeds, for the held-out operator pair named in each row (\texttt{\^{}}\,=\,XOR, \texttt{|}\,=\,OR, \texttt{\&}\,=\,AND, \texttt{d}\,=\,NAND, \texttt{L}\,=\,LSHIFT, \texttt{R}\,=\,RSHIFT). Baselines: KNN\,=\,0\%, MLP (multilayer perceptron)\,=\,0\%, KRR\,=\,0--18.2\%, Oracle\,=\,100\%. Full details in Appendix~\ref{app:symbolic}.}
\label{tab:symbolic}
\small
\begin{tabular}{lccc}
\toprule
Holdout & Full (familiar) & Full (opaque) & Label-only \\
\midrule
\texttt{\^{}|}  & 73.2 $\pm$ 2.0  & 73.0 $\pm$ 5.1  & 29.1 $\pm$ 0.8  \\
\texttt{\&L}    & 55.3 $\pm$ 3.9  & 65.9 $\pm$ 1.9  & 14.5 $\pm$ 2.5  \\
\texttt{R\^{}}  & 28.2 $\pm$ 6.6  & 24.6 $\pm$ 7.0  & 6.1 $\pm$ 3.2   \\
\texttt{d|}     & 54.4 $\pm$ 17.7 & 56.3 $\pm$ 14.2 & 33.3 $\pm$ 9.8  \\
\bottomrule
\end{tabular}
\end{table}

\section{Discussion}

\subsection{The plateau question}

A central question in artificial intelligence, with scientific, economic, and societal consequences, is whether transformer-based models are approaching a fundamental ceiling. Interpolation-only accounts provide a clear route to such a ceiling: if all a model does is recombine observed patterns, then its capabilities are bounded by those patterns, and neither scaling nor deeper chain-of-thought reasoning changes what the architecture \emph{can} do.

These are serious arguments, and they are partially right. Transformers often do default to interpolation when interpolation works \citep{mccoy2023embers,wu2024reasoning,mirzadeh2024gsm}. That is what makes stronger architectural-ceiling claims, including those advanced by \citet{marcus2022deep}, worth testing rather than dismissing. The question is whether this reflects a fundamental architectural limitation or a preference for the easier strategy. Our experiment directly tests this by constructing a task where interpolation achieves provably 0\%, not just low performance but mathematical impossibility, and showing that a standard two-layer transformer nevertheless succeeds, reaching up to 100\%.

This falsifies the strongest architectural form of interpolation-only accounts: a standard transformer can succeed where interpolation is provably impossible. To the extent that plateau arguments rest on that premise, it removes that basis for a ceiling. The architectural capacity for genuine computation beyond the patterns in training data therefore exists. We are precise about scope: this proves the architecture \emph{can} compute absent functions, not that any particular LLM \emph{does} so, nor that scaling will produce unbounded improvement.

\subsection{The mechanism: constraint propagation, not memorisation}

Three results jointly establish how the transformer computes beyond its training distribution. First, the ablation studies isolate what matters: no unrolling 63.1\%, hard unrolling 65.5\%, soft unrolling 96.7\%. The progression shows that neither seeing more data nor having multi-step structure is sufficient. The model requires differentiable signal propagation through intermediate predictions. Second, the GF(2) constraint solver proves the information exists: with two timesteps, 100\% of samples are uniquely identifiable; the constraint system is heavily overdetermined. Third, constraint density governs learning speed: radius-2 rules with 8 hidden patterns (but 24 visible) are learned in 50 epochs while radius-1 rules with 1 hidden pattern (and only 7 visible) require 220 epochs. Within a single rule, hiding \emph{more} patterns makes learning harder ($k$-sweep, Appendix~\ref{app:ksweep}), confirming that it is the visible constraints, not the hidden ones, that drive learning. The implication is constructive: richer domains---more operators, wider interactions, deeper causal chains---provide denser constraint graphs, so this mechanism may strengthen rather than weaken as task complexity grows. Appendix~\ref{app:explicit_local_boolean_block} shows that it is also built into the transformer block itself, with attention gathering local bits and the ReLU FFN computing the Boolean rule.

\subsection{Implications for LLMs}

The mechanism we identify is constraint propagation through sequential prediction. It requires only attention and feedforward layers. It is supported empirically by the symbolic benchmark, which shows the same capacity to compute beyond the training distribution in a setting unrelated to CA, and constructively by Appendix~\ref{app:explicit_local_boolean_block}, which gives an exact local-rule circuit for a standard transformer block. The model goes further than merely computing the right answer: it externalises discovered operators as tokens in a structured derivation, representing the rule in a communicable form.

We emphasise that this is shown in principle. Natural language offers far denser constraints than our benchmarks: longer proof chains, worked calculations, programs, and explanations provide many opportunities for partial token-level matching before a derivation goes wrong. But those constraints are noisier, the rules are softer, and the training signal is less consistently structured. Whether the density advantage outweighs the noise is an empirical question. Our result establishes that the barrier is not architectural; the open question is whether real training conditions produce the right dynamics.

Our phase transition result suggests \emph{when} this capacity activates: holdout accuracy remains low relative to supervised accuracy until supervised performance reaches roughly $\sim$85\%, then rises sharply. This delayed transition appears in both experiments despite their different structure. If the pattern transfers to natural language, it predicts a threshold effect: models below a competence level in some domain may rely mostly on interpolation, while those above it may begin computing rules never directly present in their training data. This has a direct bearing on the plateau debate; it suggests that new capabilities could emerge as models cross domain-specific competence thresholds, rather than improvements tapering off smoothly.

In Experiment~2, the model receives only five integers and must recover an operator composition it has never seen, expressing the result as a symbolic derivation. This is a minimal analogue of a larger possibility: an LLM's training data contains billions of observations that are consequences of underlying rules---experimental measurements shaped by physical laws, clinical outcomes governed by biological mechanisms, mathematical results following from axioms. The constraint-density result suggests that such domains may, in some respects, be \emph{easier}, not harder, for this mechanism: compared with our 7-operator benchmark, real-world systems often contain far more interacting rules and observable consequences, and may therefore provide richer indirect signal for rule recovery. Our results therefore show that transformer architectures can, in principle, recover rules absent from their training data and express them explicitly---not merely compute the right answer internally but represent the discovered rule in a communicable symbolic form. This is a prerequisite for AI systems that contribute to scientific discovery rather than merely recapitulating known patterns.

A related future direction is to test whether validators such as calculators, interpreters, proof checkers, theorem provers, simulators, or unit tests can provide semantic feedback on intermediate or final outputs. Such feedback could give models a denser training signal than exact token matching alone, analogous to how soft unrolling supplied smoother downstream signal in Experiment~1, potentially lowering the competence threshold at which constraint propagation becomes useful.

\subsection{Limitations}

We are precise about the boundaries of our claim.

\begin{enumerate}
\item \textbf{Synthetic setting.} We use two-layer transformers on synthetic tasks (1D cellular automata, integer operator chains). The value is the proof of principle, not the specific numbers; we cannot directly extrapolate to LLMs at scale.
\item \textbf{Interpolation lower bound specificity.} The provable 0\% interpolation result is specific to pure XOR / Rule~150. The broader architectural construction in Appendix~\ref{app:explicit_local_boolean_block} is more general: a standard transformer block can implement any radius-$1$ Boolean rule exactly.
\item \textbf{Bimodal outcomes.} Not every seed succeeds, though failure rates respond systematically to training conditions (more epochs, denser constraints). For Rule~150, two of eight patterns show bimodal convergence where the model finds a consistent alternative rule. Experiment~2 does not exhibit bimodality: seed variance is continuous rather than all-or-nothing, consistent with the explicit derivation structure providing more reliable gradient signal.
\item \textbf{Proof of possibility, not performance.} Our result shows that transformers are not fundamentally limited to interpolation, and Appendix~\ref{app:explicit_local_boolean_block} makes this concrete with an explicit local-rule circuit in a standard transformer block. It does not prove that any particular scaling trajectory will continue, that training on natural language produces the same dynamics, or that all domains contain sufficient indirect signal for constraint propagation. The absence of a fundamental ceiling does not imply the absence of practical ones.
\item \textbf{Symbolic benchmark scope.} The operator benchmark shows the mechanism beyond CA, but in a small deterministic setting and without a mathematical guarantee. Whether it extends to stochastic or continuous-valued domains remains untested.
\end{enumerate}

\section{Conclusion}

We constructed a setting where similarity-based interpolation is ruled out by mathematical proof and showed that a standard two-layer transformer succeeds anyway. KNN, Gaussian processes, RBF SVMs, and Random Forests all achieve provably 0\% on the held-out XOR pattern. The transformer, receiving zero direct supervision, recovers the hidden rule in 47 of 60 runs; the best models reach 100\%, verified by polynomial extraction confirming XOR structure, causal layer ablations, and identified parity neurons.

This result has direct implications for the debate over whether AI progress faces a fundamental plateau. The strongest version of the plateau argument---that the transformer architecture is inherently limited to interpolation and therefore bounded by its training data---is falsified. The architectural capacity for genuine computation exists and requires only attention and feedforward layers, the same components in every large language model. A second experiment on compositional operator chains, structurally unrelated to cellular automata, reproduces a similar phase transition and is consistent with the same mechanistic signature, showing the capacity is not specific to cellular automata. The model externalises discovered operators as learned symbols, suggesting that constraint-driven learning may enable models not only to discover absent rules but to communicate them.

Our result does not guarantee unlimited progress, nor does it prove that LLMs exercise this capacity in practice. What it does is remove the theoretical foundation for claims of an inevitable ceiling rooted in the assumption that transformers can only interpolate over their training distribution. The question shifts from ``can transformers compute beyond interpolation?'' to ``when and under what conditions do they?'' Our phase transition result suggests an answer: above a domain-specific competence threshold, constraint propagation activates and models may begin computing rules never directly present in their training data. The question of whether this capacity is realised at scale remains open, but the question of whether it is architecturally possible is now answered, both empirically and constructively.

\bibliography{references}
\bibliographystyle{plainnat}

\newpage
\appendix

\section{Proof of Theorems}
\label{app:proofs}

\begin{theorem}[Rule 150 / $n=3$: all monotonic similarity interpolation fails]
\label{thm:monotone}
Let $n=3$. Any score
\[
S(p)=\sum_{q \in T} w(d_H(p,q))\,y(q)
\]
with nonnegative weights $w(1)\geq w(2)\geq w(3)\geq 0$ satisfies
$y(p)S(p)\leq 0$. Thus the similarity score never gives positive evidence for
the correct label; with any non-favourable tie break it predicts $-y(p)$ or
ties. In particular, $k$-NN majority vote fails for every odd
$k\in\{1,3,5,7\}$.
\end{theorem}

\begin{proof}[Proof of Theorem~\ref{thm:monotone}]
In $\{0,1\}^3$, the distance shells around $p$ contain 3 points at distance~1, 3 at distance~2, and 1 at distance~3. By the parity-distance identity, the 3 distance-1 and 1 distance-3 points have label $-y(p)$ (4 total), while the 3 distance-2 points have label $y(p)$ (3 total).

(a) $k=1$ or $3$: all selected neighbours at distance~1, all $-y(p)$. $k=5$: three distance-1 ($-y(p)$) plus two distance-2 ($y(p)$), majority $-y(p)$. $k=7$: all of $T$, majority $-y(p)$ (4 vs 3).

(b) The weighted vote is
\[
S(p)=\sum_{q \in T} w(d_H(p,q))\,y(q)
= y(p)[-3w(1)+3w(2)-w(3)].
\]
Since $w(1)\geq w(2)$, we have $-3w(1)+3w(2)\leq 0$, and $-w(3)\leq0$, so
$y(p)S(p)\leq0$.
\end{proof}

\begin{theorem}[General $n$: nondegenerate completely monotone kernels give 0\%]
\label{thm:cm}
For any $n \geq 1$, let
\[
S(p)=\sum_{q\in T} w(d_H(p,q))\,y(q),
\qquad
w(d)=\int_{[0,1]} r^d\,d\mu(r),
\]
where $\mu$ is a finite nonnegative measure with $\mu((0,1])>0$. Then
$y(p)S(p)<0$, so the kernel-similarity vote predicts $-y(p)$.
\end{theorem}

\begin{proof}[Proof of Theorem~\ref{thm:cm}]
Group by Hamming distance:
\[
S(p)=y(p)\sum_{d=1}^{n}\binom{n}{d}(-1)^d w(d).
\]
Substituting the mixture representation and applying the binomial theorem,
\[
\sum_{d=1}^{n}\binom{n}{d}(-1)^d w(d)
=
\int_{[0,1]}\bigl[(1-r)^n-1\bigr]\,d\mu(r)<0,
\]
because $(1-r)^n-1<0$ for all $r\in(0,1]$ and $\mu((0,1])>0$. Therefore
$y(p)S(p)<0$.
\end{proof}

\begin{corollary}
\label{cor:rbf}
For any $\gamma>0$, the RBF kernel on $\{0,1\}^n$ satisfies
\[
\exp(-\gamma\|x-q\|^2)=(e^{-\gamma})^{d_H(x,q)}.
\]
Thus it is completely monotone as a function of Hamming distance, so
Theorem~\ref{thm:cm} applies directly. Moreover $r=e^{-\gamma}\in(0,1)$, so the
GP / kernel-ridge result of Theorem~\ref{thm:gp} applies whenever $\sigma^2>0$.
\end{corollary}

\begin{proof}[Proof of Corollary~\ref{cor:rbf}]
On $\{0,1\}^n$, $(x_i - q_i)^2 \in \{0,1\}$, so
$\|x-q\|^2 = d_H(x,q)$ and
\[
\exp(-\gamma\|x-q\|^2)=(e^{-\gamma})^{d_H(x,q)}.
\]
For $\gamma>0$, $e^{-\gamma}\in(0,1)$. Therefore the kernel is a completely
monotone function of Hamming distance for Theorem~\ref{thm:cm}, and it satisfies
the $0<r<1$ assumption of Theorem~\ref{thm:gp}.
\end{proof}

\begin{theorem}[GP / kernel ridge regression with RBF kernel]
\label{thm:gp}
Let $0<r<1$ and $\sigma^2>0$. For GP regression or kernel ridge regression with
$K(x,q)=r^{d_H(x,q)}$ and regularized covariance $C=K+\sigma^2 I$, trained on
$T=\{0,1\}^n\setminus\{p\}$, the leave-one-out prediction at $p$ has sign
$-y(p)$.
\end{theorem}

\begin{proof}[Proof of Theorem~\ref{thm:gp}]
Because $0<r<1$, the kernel matrix is positive definite. The kernel
$r^{d_H(x,q)} = \prod_{t=1}^n \kappa(x_t, q_t)$ where
$\kappa(a,b) = r^{|a-b|}$. The kernel matrix $K$ is the $n$-fold Kronecker product
of $\bigl(\begin{smallmatrix} 1 & r \\ r & 1 \end{smallmatrix}\bigr)$, with
eigenvectors formed by choosing $(1,1)$ (eigenvalue $1+r$) or $(1,-1)$
(eigenvalue $1-r$) per coordinate. The parity vector $y$ picks $(1,-1)$ in every
coordinate, giving eigenvalue $\lambda_y = (1-r)^n$---the smallest eigenvalue.

Let $K$ and $C = K + \sigma^2 I$ denote the kernel and regularized covariance
matrices on the full cube $\{0,1\}^n$. Then $C^{-1}y = cy$, where
$c = ((1-r)^n + \sigma^2)^{-1}$. The prediction at $p$ after training on
$T=\{0,1\}^n\setminus\{p\}$ is the leave-one-out prediction from the full-cube
system. By the LOO identity \citep{rasmussen2006gaussian},
\[
\mu_{-p}(p)=y(p)(1-c/\beta),
\qquad
\beta=[C^{-1}]_{pp}=\frac{1}{2^n}\sum_S\frac{1}{\lambda_S+\sigma^2}.
\]
Since $\lambda_y$ is the smallest eigenvalue, $c$ is the maximum of
$1/(\lambda_S+\sigma^2)$, and $\beta$, the average over all $2^n$ eigenvalues,
is strictly less. Thus $c/\beta>1$, so $\mu_{-p}(p)$ has sign $-y(p)$.
\end{proof}

\begin{lemma}[Decision trees and Random Forests]
\label{lem:rf}
Any axis-aligned decision tree whose nonempty leaves predict the training-set
majority, and whose empty leaves fall back to the global training-set majority,
predicts $-y(p)$. Any Random Forest aggregating such trees also predicts $-y(p)$.
\end{lemma}

\begin{proof}[Proof of Lemma~\ref{lem:rf}]
Any nonempty axis-aligned leaf is a subcube with $m$ free coordinates. For
$m \geq 1$, parity is balanced on the subcube: pairing each point with its
neighbour from flipping one free coordinate establishes a bijection that flips
parity, giving $2^{m-1}$ of each label. If the leaf contains the held-out point
$p$, then removing $p$ leaves $2^{m-1}-1$ points with label $y(p)$ and $2^{m-1}$
points with label $-y(p)$, so the majority is $-y(p)$. If the leaf does not
contain $p$, its prediction is irrelevant to the query at $p$. For $m=0$, the
query leaf is the singleton $\{p\}$, which contains no training examples because
$p\notin T$; by the stated fallback convention it predicts the global training-set
majority, which is $-y(p)$ because $T$ contains $2^{n-1}-1$ points with label
$y(p)$ and $2^{n-1}$ with label $-y(p)$. Every tree predicts $-y(p)$ at $p$, so
any majority-vote or averaging ensemble does too.
\end{proof}

\subsection{Position-conditioned extension to raw neighbourhood embeddings}
\label{app:embedding_extension}

The hard-gap arguments above extend from raw binary inputs to raw pre-transformer
neighbourhood embeddings after conditioning on centre location, or equivalently
after subtracting the slotwise positional offsets before comparing neighbourhoods.
This isolates the role of the value embedding itself. Because the empirical
neighbourhood-embedding baseline in Table~\ref{tab:interpolation} pools examples
across absolute positions, where those offsets no longer cancel, the results below
do not by themselves prove the pooled table row; they prove the corresponding
fixed-location statement.

\begin{proposition}[Fixed-position neighbourhood embeddings preserve Hamming geometry]
\label{prop:hamming}
Let $\phi(0),\phi(1)\in\mathbb R^m$ denote the two value-embedding vectors induced
by the learned input projection, and let $\pi_1,\ldots,\pi_W\in\mathbb R^m$ denote
the learned absolute position embeddings. Fix a centre location
$i\in\{2,\ldots,W-1\}$ and define the raw pre-transformer neighbourhood embedding
\[
\Phi_i(x_L,x_C,x_R)
=
\big[\phi(x_L)+\pi_{i-1},\ \phi(x_C)+\pi_i,\ \phi(x_R)+\pi_{i+1}\big]
\in\mathbb R^{3m}.
\]
Let $\Delta := \phi(1)-\phi(0)$. Then for all $x,q\in\{0,1\}^3$,
\[
\|\Phi_i(x)-\Phi_i(q)\|_2^2
=
\|\Delta\|_2^2\, d_H(x,q).
\]
Hence, if $\Delta\neq 0$, the eight embedded neighbourhoods form a scaled
isometric copy of the $3$-bit Hamming cube.
\end{proposition}

\begin{proof}
For each slot $s\in\{L,C,R\}$, the positional offset cancels in pairwise differences:
\[
\Phi_i(x)-\Phi_i(q)
=
\big[\phi(x_L)-\phi(q_L),\ \phi(x_C)-\phi(q_C),\ \phi(x_R)-\phi(q_R)\big].
\]
Since each input bit is binary, for each slot the difference is either $0$ (if the
bits agree) or $\pm \Delta$ (if they differ). Therefore
\[
\|\Phi_i(x)-\Phi_i(q)\|_2^2
=
\sum_{s\in\{L,C,R\}} \|\phi(x_s)-\phi(q_s)\|_2^2
=
\|\Delta\|_2^2 \sum_{s\in\{L,C,R\}} \mathbf 1[x_s\neq q_s]
=
\|\Delta\|_2^2 d_H(x,q).
\]
\end{proof}

\begin{corollary}[Rule~150 hard gap in fixed-position neighbourhood embeddings]
\label{cor:embedding_gap}
Let $y(x):=(-1)^{x_L+x_C+x_R}$ be the Rule~150 label, fix a held-out pattern
$p\in\{0,1\}^3$, and let $T:=\{0,1\}^3\setminus\{p\}$. Assume
$\Delta:=\phi(1)-\phi(0)\neq0$, so the binary value embedding distinguishes the
two input symbols. For any fixed centre location $i$:

\begin{enumerate}
\item Any monotone similarity score of the form
\[
S_i(p)=\sum_{q\in T} W(\|\Phi_i(p)-\Phi_i(q)\|_2)\, y(q),
\]
with $W$ nonnegative and nonincreasing, satisfies $y(p)S_i(p)\le0$. Thus the
score never gives positive evidence for the correct label; with any non-favourable
tie break it predicts $-y(p)$ or ties.

\item Any radial similarity method with
\[
K_i(x,q)=\psi(\|\Phi_i(x)-\Phi_i(q)\|_2^2),
\]
where $\psi$ is a nonzero completely monotone function, reduces to a
nondegenerate completely monotone function of Hamming distance, so
Theorem~\ref{thm:cm} applies. In particular,
\[
K_i(x,q)=\exp(-\gamma\|\Phi_i(x)-\Phi_i(q)\|_2^2)
      =\big(e^{-\gamma\|\Delta\|_2^2}\big)^{d_H(x,q)}
\]
for $\gamma>0$, so Corollary~\ref{cor:rbf} applies directly, and the GP /
kernel-ridge result of Theorem~\ref{thm:gp} applies unchanged with
$r=e^{-\gamma\|\Delta\|_2^2}\in(0,1)$.

\item Any axis-aligned decision tree on $\Phi_i$ using the leaf-majority and
empty-leaf fallback convention of Lemma~\ref{lem:rf} predicts $-y(p)$; therefore
Lemma~\ref{lem:rf} extends unchanged to Random Forests built from such trees.
\end{enumerate}
\end{corollary}

\begin{proof}
Let $\delta:=\|\Delta\|_2$. By Proposition~\ref{prop:hamming}, the seven training points around
a held-out pattern $p$ lie on shells of radii $\delta,\sqrt2\,\delta,\sqrt3\,\delta$
with counts $3,3,1$. By the parity-distance identity, the distance-$\delta$ and
distance-$\sqrt3\,\delta$ points have label $-y(p)$, while the distance-$\sqrt2\,\delta$
points have label $y(p)$. Hence
\[
S_i(p)
=
\sum_{q\in T} W(\|\Phi_i(p)-\Phi_i(q)\|_2)\, y(q)
=
y(p)\big[-3W(\delta)+3W(\sqrt2\,\delta)-W(\sqrt3\,\delta)\big].
\]
The bracketed term is nonpositive because $W$ is nonincreasing, so
$y(p)S_i(p)\le0$, proving (1).

For (2), Proposition~\ref{prop:hamming} shows that any radial kernel in
$\|\Phi_i(x)-\Phi_i(q)\|_2^2$ is a function of $d_H(x,q)$ alone. If $\psi$ is a
nonzero completely monotone function, Bernstein's theorem gives
\[
\psi(z)=\int_{[0,\infty)} e^{-tz}\,d\nu(t)
\]
for some nonzero finite nonnegative measure $\nu$. Therefore
\[
K_i(x,q)
=
\int_{[0,\infty)} \big(e^{-t\|\Delta\|_2^2}\big)^{d_H(x,q)} d\nu(t),
\]
a nonnegative mixture of powers $r^{d_H(x,q)}$. Since $\Delta\neq0$, the map
$t\mapsto r=e^{-t\|\Delta\|_2^2}$ sends $[0,\infty)$ into $(0,1]$. Because
$\nu$ is nonzero, the induced measure has positive mass on $(0,1]$, exactly the
setting of Theorem~\ref{thm:cm}. For the RBF kernel with $\gamma>0$,
$r=e^{-\gamma\|\Delta\|_2^2}\in(0,1)$, so the same RBF argument as
Corollary~\ref{cor:rbf}, with effective parameter $\gamma\|\Delta\|_2^2$, and
the GP / kernel-ridge result of Theorem~\ref{thm:gp} apply.

For (3), each coordinate of $\Phi_i$ takes at most two values, one for bit $0$
and one for bit $1$ in a fixed slot. Any nonconstant axis-aligned threshold split
is therefore equivalent to fixing one original bit value in one slot. Leaves are
subcubes of $\{0,1\}^3$, and the empty-leaf fallback convention is the same as in
Lemma~\ref{lem:rf}, so the lemma applies unchanged.
\end{proof}

\begin{remark}[Training does not break this geometry]
This statement depends only on the additive form of the input parameterisation,
not on the initialisation. SGD may move $\phi(0)$, $\phi(1)$, and the position
embeddings, but for fixed centre location the position terms cancel exactly, and
every bit flip still contributes the same squared distance $\|\Delta\|_2^2$.
Training may translate, rotate, or rescale the embedded cube, but it cannot warp
its Hamming-shell structure. The barrier disappears only after the transformer
layers mix positions nonlinearly.
\end{remark}

\begin{remark}[Why the statement is position-conditioned]
If one compares neighbourhoods centred at different absolute locations $i\neq k$,
then
\[
\Phi_i(x)-\Phi_k(q)
=
\big[\phi(x_L)-\phi(q_L)+(\pi_{i-1}-\pi_{k-1}),\
     \phi(x_C)-\phi(q_C)+(\pi_i-\pi_k),\
     \phi(x_R)-\phi(q_R)+(\pi_{i+1}-\pi_{k+1})\big],
\]
so pairwise distances acquire additional position-dependent terms. The clean
Hamming isometry therefore holds only after conditioning on centre location, or
equivalently after subtracting the slotwise positional offsets before comparison.
This is why the appendix result sharpens, but does not replace, the empirical
pooled-position neighbourhood-embedding baseline in Table~\ref{tab:interpolation}.
\end{remark}

\subsection{A single transformer block implements any radius-$1$ Boolean rule}
\label{app:explicit_local_boolean_block}

We now strengthen the Rule~150 construction. Rather than proving a circuit only for
parity, we show that a single standard transformer encoder block already implements
\emph{any} radius-$1$ Boolean rule
\[
f:\{0,1\}^3 \to \{0,1\},
\qquad
y_i = f(x_{i-1},x_i,x_{i+1}),
\]
exactly at the final logit / threshold level. Rule~150 is the special case
\(f(L,C,R)=L\oplus C\oplus R\).

The proof separates the two roles:
\begin{enumerate}
\item self-attention performs smooth local routing of the left and right neighbour
bits, while the centre bit is already present in the residual stream;
\item the ReLU FFN performs the non-affine Boolean rule computation.
\end{enumerate}

\paragraph{Canonical input code.}
Let
\[
s(b):=2b-1\in\{-1,+1\},
\qquad
\theta_j:=\frac{2\pi j}{W}.
\]
For each position \(j\), define the antisymmetric absolute-position code
\[
p_j:=(\cos\theta_j,\sin\theta_j)\in\mathbb R^2,
\qquad
\pi_j:=(\cos\theta_j,-\cos\theta_j,\sin\theta_j,-\sin\theta_j)\in\mathbb R^4.
\]
We use the \(12\)-dimensional token code
\[
h_j^{(0)}
=
\bigl(0,0,\ s(x_j),-s(x_j),\ 0,0,\ \pi_j,\ 0,0\bigr)\in\mathbb R^{12},
\]
grouped as
\[
(\text{left pair},\ \text{centre pair},\ \text{right pair},\ \text{position code},\ \text{readout pair}).
\]
This canonical input code is realizable by the actual input layer: choose the input
affine map \(e:\{0,1\}\to\mathbb R^{12}\) as
\[
e(x_j)=(0,0,s(x_j),-s(x_j),0,0,0,0,0,0,0,0),
\]
and the absolute position embedding as
\[
\pi(j)=(0,0,0,0,0,0,\cos\theta_j,-\cos\theta_j,\sin\theta_j,-\sin\theta_j,0,0),
\]
so that \(h_j^{(0)}=e(x_j)+\pi(j)\).

\begin{lemma}[Softmax heads can route a fixed relative neighbour with arbitrarily high mass]
\label{lem:softmax_relative_mass}
Fix width \(W\ge 3\), periodic positions \(j\in\{0,\dots,W-1\}\), and the
circle-code position coordinates \(p_j=(\cos\theta_j,\sin\theta_j)\). For each
offset \(\tau\in\{-1,+1\}\) and every \(\rho\in(0,1)\), there exist query and key
projections for one attention head such that, at every position \(i\), the
attention weight on \(i+\tau\) is at least \(\rho\).
\end{lemma}

\begin{proof}
For offset \(\tau\), take the query at position \(i\) proportional to
\(p_{i+\tau}\) and the key at position \(j\) proportional to \(p_j\). Both maps
ignore the content coordinates and read only the position-code coordinates. Then
the scaled dot-product score has the form
\[
s^{(\tau)}_{ij}=\alpha\,\cos(\theta_j-\theta_{i+\tau})
\]
for some free scale \(\alpha>0\); the standard factor \(1/\sqrt{d_h}\) is absorbed into
\(\alpha\). The score is uniquely maximized at \(j=i+\tau\pmod W\). Since the position
set is finite, there is a uniform positive gap between the target score and the
largest non-target score. Therefore, as \(\alpha\to\infty\), the softmax mass on
\(i+\tau\) tends to \(1\) uniformly in \(i\), so any \(\rho<1\) can be achieved.
\end{proof}

\begin{lemma}[Weighted averages preserve sign margins under a target-mass condition]
\label{lem:weighted_margin}
Let \(0<m\le M\), let \(v_1,\dots,v_n\in[-M,-m]\cup[m,M]\), and let
\(a_1,\dots,a_n\ge 0\) with \(\sum_j a_j=1\). Fix an index \(t\) and suppose
\(a_t\ge \rho\). If
\[
\rho m-(1-\rho)M > 0,
\]
then the weighted average
\[
\bar v:=\sum_{j=1}^n a_j v_j
\]
has the same sign as \(v_t\), and moreover
\[
|\bar v|\ge \rho m-(1-\rho)M.
\]
In particular, if \(m=M=1\), then any \(\rho>1/2\) yields the margin
\(|\bar v|\ge 2\rho-1\).
\end{lemma}

\begin{proof}
If \(v_t\ge m\), then the smallest possible value of \(\bar v\) occurs when every
non-target term equals \(-M\), giving
\[
\bar v\ge a_t m-(1-a_t)M\ge \rho m-(1-\rho)M>0.
\]
If \(v_t\le -m\), then the largest possible value of \(\bar v\) occurs when every
non-target term equals \(+M\), giving
\[
\bar v\le -a_t m+(1-a_t)M\le -\bigl(\rho m-(1-\rho)M\bigr)<0.
\]
The claimed lower bound on \(|\bar v|\) follows.
\end{proof}

\begin{lemma}[LayerNorm on antisymmetric pairs]
\label{lem:ln_antisymmetric_pairs}
Let \(\mathrm{LN}\) be LayerNorm with gain \(\gamma=\mathbf 1\), bias \(\beta=\mathbf 0\),
and stabilizer \(\varepsilon_{\mathrm{LN}}>0\). For any
\[
z=(a,-a,b,-b,c,-c,d,-d,e,-e,u,-u)\in\mathbb R^{12},
\]
one has
\[
\mathrm{LN}(z)=\eta(z)\,z
\]
for some scalar \(\eta(z)>0\).
\end{lemma}

\begin{proof}
The mean is zero and the variance is
\[
\frac{a^2+b^2+c^2+d^2+e^2+u^2}{6}.
\]
Substituting into the LayerNorm formula gives
\[
\eta(z)=\left(\frac{a^2+b^2+c^2+d^2+e^2+u^2}{6}+\varepsilon_{\mathrm{LN}}\right)^{-1/2}>0,
\]
where positivity follows from \(\varepsilon_{\mathrm{LN}}>0\).
\end{proof}

\begin{lemma}[A width-$8$ ReLU FFN computes any Boolean rule from sign margins]
\label{lem:relu_any_boolean_rule}
Let \(0<m\le M<2m\). Suppose real numbers \(z_L,z_C,z_R\) satisfy
\[
z_S\in[m,M]\quad\text{when }S=1,
\qquad
z_S\in[-M,-m]\quad\text{when }S=0,
\]
for some \((L,C,R)\in\{0,1\}^3\). Choose any threshold \(\theta\) with
\[
2M-m<\theta<3m.
\]
For each sign pattern \(\sigma=(\sigma_L,\sigma_C,\sigma_R)\in\{-1,+1\}^3\), define
\[
u_\sigma
=
\operatorname{ReLU}(\sigma_L z_L+\sigma_C z_C+\sigma_R z_R-\theta).
\]
Then exactly one hidden unit is positive, namely the one with
\[
\sigma=(s(L),s(C),s(R)).
\]
Moreover, for that active unit one has
\[
u_\sigma\ge 3m-\theta>0.
\]
Consequently, for any Boolean rule \(f:\{0,1\}^3\to\{0,1\}\), if we write
\[
b(\sigma):=\left(\frac{\sigma_L+1}{2},\frac{\sigma_C+1}{2},\frac{\sigma_R+1}{2}\right)
\]
and define
\[
g_f(z_L,z_C,z_R)
:=
\sum_{\sigma\in\{-1,+1\}^3}
\bigl(2f(b(\sigma))-1\bigr)\,u_\sigma,
\]
then
\[
g_f(z_L,z_C,z_R)>0 \iff f(L,C,R)=1,
\qquad
g_f(z_L,z_C,z_R)<0 \iff f(L,C,R)=0.
\]
\end{lemma}

\begin{proof}
Let \(\sigma^\star=(s(L),s(C),s(R))\). Then
\[
\sigma^\star_L z_L+\sigma^\star_C z_C+\sigma^\star_R z_R\ge 3m,
\]
so \(u_{\sigma^\star}\ge 3m-\theta>0\). If \(\sigma\neq\sigma^\star\), then at least one
sign is wrong, so the corresponding signed sum contains at most two terms from
\([m,M]\) and at least one term from \([-M,-m]\). Hence
\[
\sigma_L z_L+\sigma_C z_C+\sigma_R z_R\le 2M-m<\theta,
\]
so \(u_\sigma=0\). Therefore exactly one detector is active, and \(g_f\) takes the sign
of the desired output label.
\end{proof}

\begin{theorem}[A single standard transformer block implements any radius-$1$ Boolean rule]
\label{thm:any_local_rule_single_block}
Fix width \(W\ge3\) with periodic boundary conditions. Consider a post-LN
transformer encoder block in evaluation mode with learned absolute position
embeddings, multi-head softmax self-attention, residual connections, LayerNorm, a
ReLU FFN, and a final affine scalar readout. Assume model dimension at least
\(12\), at least \(2\) attention heads, and FFN width at least \(8\). Then for
every Boolean rule \(f:\{0,1\}^3\to\{0,1\}\), there exists a parameter setting such
that for every binary input row \(x\in\{0,1\}^W\) and every position \(i\),
\[
\ell_i(x)>0
\quad\Longleftrightarrow\quad
f(x_{i-1},x_i,x_{i+1})=1,
\]
with indices interpreted modulo \(W\). Thus thresholding the final logit at \(0\)
computes the radius-$1$ rule \(f\) exactly.
\end{theorem}

\begin{proof}
Use the canonical input code \(h_j^{(0)}\) above and fix any \(\rho\in(1/2,1)\), to be
chosen sufficiently close to \(1\).

Use two attention heads in the first sublayer. The left head uses
Lemma~\ref{lem:softmax_relative_mass} with offset \(-1\), and the right head uses it
with offset \(+1\). In both heads, the query and key maps ignore content and read only
the position-code coordinates, while the value projection reads only the centre sign
pair \((s(x_j),-s(x_j))\). The multi-head output projection writes the left-head output
into the left pair and the right-head output into the right pair. The residual stream
already contains the centre pair and the position code. Thus, after the
attention-residual sublayer, the token at position \(i\) has the form
\[
r_i
=
(\alpha_i,-\alpha_i,\ s(x_i),-s(x_i),\ \beta_i,-\beta_i,\ \pi_i,\ 0,0),
\]
where
\[
\alpha_i=\sum_j a^{(-1)}_{ij}s(x_j),
\qquad
\beta_i=\sum_j a^{(+1)}_{ij}s(x_j),
\]
and the target weights satisfy
\[
a^{(-1)}_{i,i-1}\ge \rho,
\qquad
a^{(+1)}_{i,i+1}\ge \rho.
\]
Applying Lemma~\ref{lem:weighted_margin} with \(m=M=1\), we obtain
\[
\operatorname{sgn}(\alpha_i)=s(x_{i-1}),
\qquad
\operatorname{sgn}(\beta_i)=s(x_{i+1}),
\qquad
|\alpha_i|,|\beta_i|\ge 2\rho-1.
\]

Choose the first LayerNorm gains to be \(\mathbf 1\) and biases to be \(\mathbf 0\).
By Lemma~\ref{lem:ln_antisymmetric_pairs}, the post-attention LayerNorm output is
\[
z_i=\lambda_i r_i
\]
for some scalar \(\lambda_i>0\). Therefore the three readable coordinates
\[
z_L:=z_{i,1},
\qquad
z_C:=z_{i,3},
\qquad
z_R:=z_{i,5}
\]
have the correct signs for \((x_{i-1},x_i,x_{i+1})\).

Because
\[
(\alpha_i,s(x_i),\beta_i)
\in
([-(1),-(2\rho-1)]\cup[2\rho-1,1])\times\{-1,+1\}\times([-(1),-(2\rho-1)]\cup[2\rho-1,1]),
\]
and because \(\lambda_i\) is a continuous positive function of these values,
compactness yields constants \(0<m\le M<\infty\) such that
\[
z_S\in[m,M]\quad\text{when }S=1,
\qquad
z_S\in[-M,-m]\quad\text{when }S=0.
\]
Moreover, as \(\rho\to 1\), the admissible triples \((\alpha_i,s(x_i),\beta_i)\)
converge uniformly to the exact codebook \((\pm1,\pm1,\pm1)\). The corresponding
pre-LayerNorm vectors therefore converge uniformly to a finite codebook with constant
squared norm, so the LayerNorm scaling factors converge uniformly to a common
positive constant. Hence \(M/m\to 1\) as \(\rho\to 1\), and for \(\rho\) sufficiently
close to \(1\) we have \(M<2m\).

Now apply Lemma~\ref{lem:relu_any_boolean_rule} to \((z_L,z_C,z_R)\). This gives a
width-$8$ hidden layer whose scalar output \(g_i\) satisfies
\[
g_i>0 \iff f(x_{i-1},x_i,x_{i+1})=1,
\qquad
g_i<0 \iff f(x_{i-1},x_i,x_{i+1})=0.
\]
Choose the FFN output projection to write only into the readout pair:
\[
(0,0,0,0,0,0,0,0,0,0,Bg_i,-Bg_i)
\]
for some fixed \(B>0\). After the FFN residual, the pre-final-LayerNorm token is
\[
v_i
=
(z_L,-z_L,z_C,-z_C,z_R,-z_R,\lambda_i\pi_i,Bg_i,-Bg_i),
\]
which is again an antisymmetric-pair vector. Choosing the final LayerNorm gain and bias
as \(\mathbf 1\) and \(\mathbf 0\), Lemma~\ref{lem:ln_antisymmetric_pairs} implies that
this final LayerNorm multiplies \(v_i\) by a positive scalar. Therefore the sign of the
readout pair is still the sign of \(g_i\). A final linear readout taking the difference
of the two readout coordinates gives the scalar logit \(\ell_i\), and hence
\[
\ell_i(x)>0 \iff f(x_{i-1},x_i,x_{i+1})=1.
\]
This proves the claim.
\end{proof}

\begin{corollary}[Rule~150 as a special case]
\label{cor:rule150_single_block}
Taking \(f(L,C,R)=L\oplus C\oplus R\) in
Theorem~\ref{thm:any_local_rule_single_block} yields an explicit exact Rule~150
circuit. Moreover, for Rule~150 the FFN width can be reduced from \(8\) to \(4\),
since only the four odd-parity patterns \(100,010,001,111\) require positive
detectors; a strictly negative final-readout bias separates zero even-parity
detector output from positive odd-parity output.
\end{corollary}

\begin{remark}[Interpretation]
Theorem~\ref{thm:any_local_rule_single_block} is a block-level architectural
existence result. A standard transformer block already contains a complete
beyond-interpolation local-rule primitive: softmax attention performs smooth local
routing, and the ReLU FFN performs the non-affine Boolean computation. Experiment~1
uses embedding dimension \(64\) and \(4\) heads, so the result applies a fortiori to
the architecture class studied in the main experiments. This is a capacity statement
about the architecture, not a claim that SGD must recover the same parameters in
every trained model.
\end{remark}

\subsection{Composable local-rule circuits across depth}
\label{app:compositional_local_rule_circuits}

Theorem~\ref{thm:any_local_rule_single_block} is the complete constructive result
used for the main claim. The following optional extension records why the same
routing-plus-FFN primitive can be re-used across blocks; it is included to clarify
the architectural interpretation, not as a dependency of the main empirical or
theoretical results.

For this extension it is convenient to use a larger, induction-friendly absolute-position
code. Fix depth \(T\ge 1\) and width \(W\ge 3\). Let
\[
u_i^{(0)}:=x_i,
\qquad
u_i^{(t)}:=f_t\bigl(u_{i-1}^{(t-1)},u_i^{(t-1)},u_{i+1}^{(t-1)}\bigr)
\qquad (t=1,\dots,T),
\]
for arbitrary Boolean rules \(f_t:\{0,1\}^3\to\{0,1\}\).

We use the signed one-hot position code
\[
q_j:=(e_j,-e_j)\in\mathbb R^{2W},
\]
where \(e_j\) is the \(j\)-th standard basis vector of \(\mathbb R^W\). Group the model
coordinates as
\[
(\text{position code},\ \text{left scratch pair},\ \text{state pair},\ \text{right scratch pair})
\in \mathbb R^{2W}\oplus\mathbb R^2\oplus\mathbb R^2\oplus\mathbb R^2.
\]
Thus model dimension \(2W+6\) suffices. Again this code is realizable by the actual
input layer: choose the input affine map
\[
e_{\mathrm{comp}}(x_i)=(0_{2W},0,0,s(x_i),-s(x_i),0,0)
\]
and the absolute position embedding
\[
\pi_{\mathrm{comp}}(i)=(q_i,0,0,0,0,0,0),
\]
so that the compositional input token equals \(e_{\mathrm{comp}}(x_i)+\pi_{\mathrm{comp}}(i)\).

\begin{lemma}[A width-$14$ ReLU FFN can update the state and clear the scratch pairs]
\label{lem:relu_update_and_clear}
Assume \((z_L,z_C,z_R)\) satisfy the sign-margin hypothesis of
Lemma~\ref{lem:relu_any_boolean_rule}. Then for any Boolean rule
\(f:\{0,1\}^3\to\{0,1\}\) and any scalar \(B>0\), there exists a one-hidden-layer ReLU
FFN of width \(14\) whose output simultaneously:
\begin{enumerate}
\item writes \((-z_L,+z_L)\) into the left scratch pair,
\item writes \((-z_C+B g_f(z_L,z_C,z_R),\ z_C-B g_f(z_L,z_C,z_R))\) into the state pair,
\item writes \((-z_R,+z_R)\) into the right scratch pair,
\end{enumerate}
where \(g_f\) is the rule logit from Lemma~\ref{lem:relu_any_boolean_rule}. Consequently,
when this FFN output is added via the residual connection to a token whose three relevant
pairs are \((z_L,-z_L)\), \((z_C,-z_C)\), and \((z_R,-z_R)\), it clears the two scratch
pairs, replaces the old state pair by \((B g_f(z_L,z_C,z_R),-B g_f(z_L,z_C,z_R))\), and
leaves the sign of the new state equal to the rule output.
\end{lemma}

\begin{proof}
Use the eight hidden units \(u_\sigma\) from Lemma~\ref{lem:relu_any_boolean_rule} to
compute \(g_f\). In addition, use the six sign-split units
\[
\operatorname{ReLU}(z_L),\ \operatorname{ReLU}(-z_L),\ \operatorname{ReLU}(z_C),\ \operatorname{ReLU}(-z_C),\ \operatorname{ReLU}(z_R),\ \operatorname{ReLU}(-z_R).
\]
From these one recovers
\[
z_L=\operatorname{ReLU}(z_L)-\operatorname{ReLU}(-z_L),
\qquad
z_C=\operatorname{ReLU}(z_C)-\operatorname{ReLU}(-z_C),
\qquad
z_R=\operatorname{ReLU}(z_R)-\operatorname{ReLU}(-z_R).
\]
A linear output projection can therefore emit the three pairwise cancellation terms
together with the fresh rule-dependent state pair.
\end{proof}

\begin{proposition}[Composable depth-$T$ local-rule circuit]
\label{prop:compositional_local_rules_depth}
For every fixed width \(W\), depth \(T\), and Boolean rules
\(f_1,\dots,f_T:\{0,1\}^3\to\{0,1\}\), there exists a post-LN \(T\)-block transformer
encoder in evaluation mode with learned absolute position embeddings, multi-head
softmax self-attention, residual connections, LayerNorm, ReLU FFNs, and a final
affine scalar readout such that, for every binary input row \(x\in\{0,1\}^W\) and every
position \(i\), the final logit satisfies
\[
\ell_i(x)>0
\quad\Longleftrightarrow\quad
u_i^{(T)}=1.
\]
It is enough to use model dimension \(2W+6\), two attention heads per block, and FFN
width \(14\) per block.
\end{proposition}

\begin{proof}[Proof sketch]
Initialize token \(i\) as
\[
h_i^{(0)}=(q_i,\ 0,0,\ s(x_i),-s(x_i),\ 0,0).
\]
We maintain the invariant that after block \(t\) and its final LayerNorm, token \(i\)
has the form
\[
h_i^{(t)}=(\lambda_{i,t} q_i,\ 0,0,\ \gamma_{i,t}s(u_i^{(t)}),-\gamma_{i,t}s(u_i^{(t)}),\ 0,0),
\]
where \(\lambda_{i,t}>0\) and \(\gamma_{i,t}\in[m_t,M_t]\) for some constants
\(0<m_t\le M_t<\infty\). Moreover, at each stage we choose the block parameters so that
\(M_t/m_t<2\).

The base case \(t=0\) is immediate with \(m_0=M_0=1\). Assume the invariant at stage
\(t-1\). In block \(t\), both queries and keys depend only on the position subspace.
Let \(P_{-1}\) and \(P_{+1}\) be the cyclic shift operators on \(\mathbb R^W\). Since
\[
P_{\tau}(e_i,-e_i)=(e_{i+\tau},-e_{i+\tau})
\qquad (\tau\in\{-1,+1\}),
\]
a left-routing head can query \(q_{i-1}\) and a right-routing head can query
\(q_{i+1}\). The position code survives every LayerNorm up to positive tokenwise
rescaling, so each head still has a unique target position in every row. Because the
possible LayerNorm scales form a compact positive set, choosing the head scale large
enough makes the target weight at least \(\rho_t\) uniformly in \(i\); again, the
standard factor \(1/\sqrt{d_h}\) is absorbed into this free head scale.

Let both value projections read only the current state pair. The multi-head output
projection writes the left head into the left scratch pair and the right head into the
right scratch pair. By Lemma~\ref{lem:weighted_margin}, if
\[
\rho_t m_{t-1}-(1-\rho_t)M_{t-1}>0,
\]
then the two routed scratch pairs have the correct neighbour signs and positive margin.
The residual stream already carries the centre state pair. After the post-attention
LayerNorm, the three readable scalars \((z_L,z_C,z_R)\) therefore satisfy the sign-margin
hypothesis of Lemma~\ref{lem:relu_any_boolean_rule} for suitable constants
\(0<m_t^{\mathrm{in}}\le M_t^{\mathrm{in}}<\infty\). By choosing \(\rho_t\) sufficiently
close to \(1\), one can ensure \(M_t^{\mathrm{in}}<2m_t^{\mathrm{in}}\).

Now apply Lemma~\ref{lem:relu_update_and_clear} with rule \(f_t\). The FFN clears the
left and right scratch pairs, cancels the old state pair, and writes a fresh state pair
\((B_t g_{i,t},-B_t g_{i,t})\), where the sign of \(g_{i,t}\) corresponds to
\(s(u_i^{(t)})\). Thus, before the final LayerNorm in block \(t\), the token has the form
\[
(\widetilde\lambda_{i,t} q_i,\ 0,0,\ B_t g_{i,t},-B_t g_{i,t},\ 0,0).
\]
The final LayerNorm again rescales this vector by a positive scalar, so the sign of the
new state pair is preserved.

Finally, because \(g_{i,t}\) is a continuous nonzero function on a compact domain, there
exist bounds \(0<\delta_t\le |g_{i,t}|\le \Delta_t<\infty\). The normalized state-pair
magnitude after the final LayerNorm is a continuous function of \(|g_{i,t}|\)
on \([\delta_t,\Delta_t]\). As \(B_t\to\infty\), this normalized magnitude converges
uniformly to a positive constant, so by choosing \(B_t\) sufficiently large we can make
the interval \([m_t,M_t]\) satisfy \(M_t/m_t<2\). This closes the induction.

After block \(T\), the state-pair sign corresponds to \(s(u_i^{(T)})\). A final linear readout
that takes the difference of the two state coordinates yields
\[
\ell_i(x)>0 \iff u_i^{(T)}=1.
\]
\end{proof}

\begin{remark}[Discussion]
Theorem~\ref{thm:any_local_rule_single_block} shows that a standard transformer block
already contains a complete local beyond-interpolation primitive: attention performs
smooth neighbour routing, while the FFN performs the non-affine rule computation.
Proposition~\ref{prop:compositional_local_rules_depth} is a broader capacity result,
showing that such primitives can be stacked across layers. This is not a claim that
every trained model realizes this exact circuit; it is an architectural existence
statement about what standard transformer components can do.
\end{remark}

\section{Circuit Analysis Details}
\label{app:circuits}

\textbf{Fitted polynomial coefficients.} For two independently trained Rule~150 models achieving 100\% holdout accuracy:

\begin{table}[htbp]
\centering
\small
\begin{tabular}{lrrrrrrrr}
\toprule
Model & $c_0$ & $c_L$ & $c_C$ & $c_R$ & $c_{LC}$ & $c_{LR}$ & $c_{CR}$ & $c_{LCR}$ \\
\midrule
Seed 4444 & $-5.89$ & $11.82$ & $7.13$ & $11.63$ & $-19.28$ & $-23.60$ & $-18.91$ & $42.88$ \\
Seed 456 & $-6.50$ & $12.59$ & $8.26$ & $12.91$ & $-20.39$ & $-25.27$ & $-21.11$ & $46.14$ \\
\bottomrule
\end{tabular}
\end{table}

\textbf{Layer ablation.}

\begin{table}[htbp]
\centering
\small
\begin{tabular}{lcc}
\toprule
Layer Ablated & Seed 4444 & Seed 456 \\
\midrule
Layer 0 & 20.5\% (79.5\% drop) & 13.1\% (86.9\% drop) \\
Layer 1 & 0.0\% (100\% drop) & 0.0\% (100\% drop) \\
\bottomrule
\end{tabular}
\end{table}

\textbf{Logit lens.}

\begin{table}[htbp]
\centering
\small
\begin{tabular}{lcc}
\toprule
Layer & Seed 4444 & Seed 456 \\
\midrule
Embedding & 1.2\% & 2.7\% \\
After layer 0 & 0.0\% & 0.0\% \\
After layer 1 & 100.0\% & 100.0\% \\
\bottomrule
\end{tabular}
\end{table}

\textbf{Linear probing.} Linear classifiers decoding individual bit values and XOR from each layer's activations:

\begin{table}[htbp]
\centering
\small
\begin{tabular}{lcccc}
\toprule
Layer & Left & Centre & Right & XOR \\
\midrule
Embedding & 51.6\% & 100.0\% & 51.5\% & 51.6\% \\
After layer 0 & 75.4\% & 100.0\% & 75.3\% & 56.8\% \\
After layer 1 & 76.8\% & 99.2\% & 75.8\% & 98.4\% \\
\bottomrule
\end{tabular}
\end{table}

\textbf{Activation patching.} Corrupt inputs by flipping neighbour bits (accuracy $\to \sim$6\%), then patch clean activations:

\begin{table}[htbp]
\centering
\small
\begin{tabular}{lcc}
\toprule
Patch Point & Seed 4444 Recovery & Seed 456 Recovery \\
\midrule
Embedding & 98.6\% & 98.6\% \\
After layer 0 & 98.6\% & 100.0\% \\
After layer 1 & 100.0\% & 100.0\% \\
\bottomrule
\end{tabular}
\end{table}

\textbf{Parity neurons.} In layer~1's FFN, top parity neurons show activation differences up to 1.0 (neuron~20 in seed~456: mean activation 1.00 for XOR=1, 0.00 for XOR=0).

\section{Temporal Interpolation Baseline}
\label{app:temporal}

Random Forests and KNN given the same input as the transformer (full $t=0$ state, 101 values, plus target position and timestep). Trained on all visible cells at $t=1$--$t=4$. Tested on hidden pattern positions at $t=1$.

\begin{table}[htbp]
\centering
\small
\begin{tabular}{lcc}
\toprule
Method & Hidden Pattern [2] & Hidden Pattern [5] \\
\midrule
Random Forest (500 trees) & 0.2\% & 0.4\% \\
Random Forest (100 trees) & 8.4\% & 11.8\% \\
KNN ($k$=1) & 42.5\% & 42.1\% \\
KNN ($k$=5) & 36.2\% & 36.1\% \\
\bottomrule
\end{tabular}
\end{table}

\section{$k$-Sweep Detailed Results}
\label{app:ksweep}

\begin{figure}[htbp]
  \centering
  \includegraphics[width=0.9\textwidth]{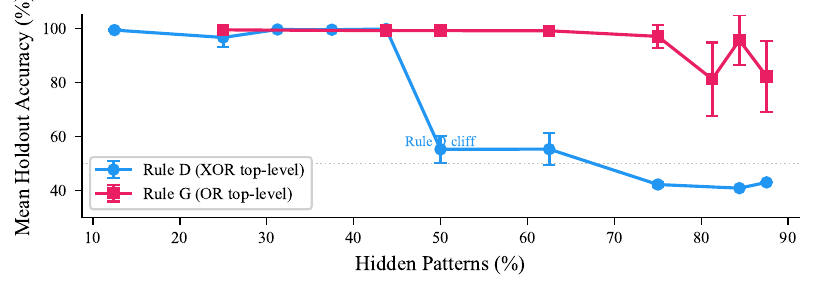}
  \vspace{-0.5em}
  \caption{Mean holdout accuracy vs.\ fraction of hidden patterns for Rule~D and Rule~G. Rule~D shows a sharp cliff at $\sim$50\% hidden; Rule~G degrades gradually, tolerating 88\% hidden. The frontier depends on intrinsic rule complexity.}
  \label{fig:k_sweep}
\end{figure}

\begin{table}[htbp]
\centering
\caption{Rule~D $k$-sweep (soft unrolling, 50 epochs, 10 seeds each).}
\small
\begin{tabular}{cccc}
\toprule
$k$ & Hidden \% & Mean Accuracy & Successes \\
\midrule
4 & 12\% & 99.4\% & 10/10 \\
8 & 25\% & 96.7\% & 10/10 \\
10 & 31\% & 99.6\% & 10/10 \\
12 & 38\% & 99.6\% & 10/10 \\
14 & 44\% & 99.8\% & 10/10 \\
16 & 50\% & 55.3\% & 0/10 \\
20 & 62\% & 55.3\% & 0/10 \\
24 & 75\% & 42.2\% & 0/10 \\
28 & 88\% & 43.0\% & 0/10 \\
\bottomrule
\end{tabular}
\end{table}

\begin{table}[htbp]
\centering
\caption{Rule~G $k$-sweep (soft unrolling, 50 epochs, 10 seeds each).}
\small
\begin{tabular}{cccc}
\toprule
$k$ & Hidden \% & Mean Accuracy & Successes \\
\midrule
14 & 44\% & 99.3\% & 10/10 \\
16 & 50\% & 99.2\% & 10/10 \\
20 & 62\% & 99.2\% & 10/10 \\
24 & 75\% & 97.1\% & 10/10 \\
26 & 81\% & 81.3\% & 7/10 \\
27 & 84\% & 95.7\% & 9/10 \\
28 & 88\% & 82.3\% & 8/10 \\
\bottomrule
\end{tabular}
\end{table}

\subsection{Cross-rule comparison}
\label{app:crossrule}

Table~\ref{tab:crossrule} reports holdout accuracy for the radius-2 rules; Table~\ref{tab:crossrule_r1} gives a per-pattern breakdown for the three radius-1 rules. All experiments use soft unrolling with masking at all timesteps and hide $k$=1 of 8 patterns.

\begin{table}[htbp]
\centering
\caption{Radius-2 rules ($k$=8 of 32 hidden). All use soft unrolling, 50 epochs.}
\label{tab:crossrule}
\small
\begin{tabular}{llccl}
\toprule
Rule & Formula & Seeds & Mean Acc.\ & Successes \\
\midrule
Rule D & $L_1 \oplus ((C \lor R_1) \land (L_2 \lor R_2))$ & 10 & 96.7\% & 10/10 \\
Rule G & $(L_1 \oplus (C \lor R_1)) \lor (L_2 \land R_2)$ & 10 & 99.5\% & 10/10 \\
\bottomrule
\end{tabular}
\end{table}

\begin{table}[htbp]
\centering
\caption{Per-pattern holdout accuracy for radius-1 rules ($k$=1 of 8 hidden, 10 seeds each). Cells show seeds achieving ${\geq}70\%$ holdout accuracy out of 10. Rule~150 (220\,ep) excludes patterns~2 and~5, which partially succeeded at 50 epochs and were not re-run.}
\label{tab:crossrule_r1}
\small
\begin{tabular}{c ccc c}
\toprule
 & \multicolumn{3}{c}{50 epochs} & 220 epochs \\
\cmidrule(lr){2-4} \cmidrule(lr){5-5}
Pattern & Rule 150 & Rule 30 & Rule 106 & Rule 150 \\
\midrule
0 & 0/10 & 1/10 & \textbf{10/10} & 3/10 \\
1 & 0/10 & 5/10 & \textbf{10/10} & \textbf{10/10} \\
2 & 4/10 & 6/10 & 0/10 & --- \\
3 & 0/10 & \textbf{10/10} & 8/10 & \textbf{10/10} \\
4 & 0/10 & 0/10 & 9/10 & \textbf{10/10} \\
5 & 4/10 & 2/10 & \textbf{10/10} & --- \\
6 & 0/10 & \textbf{10/10} & 4/10 & \textbf{10/10} \\
7 & 0/10 & 7/10 & 4/10 & 4/10 \\
\midrule
Overall & 8/80 & 41/80 & 55/80 & 47/60 \\
\bottomrule
\end{tabular}
\end{table}

The per-pattern view reveals that success is pattern-specific, not rule-specific. At 50 epochs, Rule~30 and Rule~106 each have patterns that succeed on every seed (e.g.\ Rule~30 patterns~3 and~6 at 10/10; Rule~106 patterns~0, 1, and~5 at 10/10), while Rule~150 has none. Extended training (220 epochs) allows Rule~150 to reach 47/60; patterns that showed no signal at 50 epochs emerge with longer optimisation.

\begin{table}[htbp]
\centering
\caption{Rule~30 with hard/STE unrolling (50 epochs, 10 seeds per pattern).
Successes are seeds achieving ${\geq}70\%$ holdout accuracy. The 3-NN column
states whether a 3-nearest-neighbour majority vote over the other seven
patterns predicts the held-out pattern's label correctly; for the six
patterns where it is wrong, recovery cannot be attributed to
similarity-based interpolation.}
\label{tab:rule30_hard}
\small
\begin{tabular}{cccc}
\toprule
Pattern & 3-NN vote & Successes & Best holdout \\
\midrule
0 & wrong   & 0/10 & 0.0\%   \\
1 & wrong   & 2/10 & 100.0\% \\
2 & wrong   & 4/10 & 100.0\% \\
3 & correct & 5/10 & 100.0\% \\
4 & wrong   & 0/10 & 36.2\%  \\
5 & wrong   & 2/10 & 89.4\%  \\
6 & wrong   & 2/10 & 88.6\%  \\
7 & correct & 1/10 & 91.0\%  \\
\midrule
Wrong-vote subtotal & & 10/60 & \\
Overall & & 16/80 & \\
\bottomrule
\end{tabular}
\end{table}

Recovery under hard/STE unrolling is far less reliable than under soft
unrolling, but it does occur with no soft feedback anywhere in training: on
the six patterns whose nearest-neighbour vote is wrong, 10 of 60 seeds
succeed, several reaching 100\%. Soft unrolling improves the reliability of
recovery; it is not a prerequisite for it.

The bimodal, pattern-dependent success rates suggest that learnability depends on how well each hidden pattern is constrained by its seven visible neighbours in the specific rule's truth table, not on a single global property of the rule. Some patterns are fully determined by indirect constraints and learned reliably; others are underdetermined and rarely or never recovered. The $k$-sweep (Appendix~\ref{app:ksweep}) shows the complementary effect: hiding more patterns simultaneously removes constraints, eventually crossing a threshold where no pattern can be recovered.

\section{GF(2) Constraint Solver Results}
\label{app:gf2}

\begin{table}[htbp]
\centering
\small
\begin{tabular}{lcccc}
\toprule
Timesteps & Identifiability & Mean Unknowns & Mean Constraints & Mean Rank \\
\midrule
1 & 0\% & 12.5 & 0.0 & 0.0 \\
2 & 100\% & 12.5 & 31.4 & 12.5 \\
3 & 100\% & 12.5 & 58.8 & 12.5 \\
4 & 100\% & 12.5 & 101.2 & 12.5 \\
\bottomrule
\end{tabular}
\end{table}

\section{Symbolic Operator Benchmark Details}
\label{app:symbolic}

\subsection{Task and architecture}

The task uses compositional chains of two binary operators over 6-bit integers $[0, 63]$. Seven operators are defined: XOR ($\oplus$), OR ($\lor$), AND ($\land$), NOR ($\overline{\lor}$), NAND ($\overline{\land}$), LSHIFT ($\ll_1$, left-shift by 1, masked to 6 bits), and RSHIFT ($\gg_1$, right-shift by 1). Given five input integers $(a,b,c,d,u)$ satisfying $(a \;\text{op}_1\; b) \;\text{op}_2\; c = d$, the pair $(c,d)$ helps identify op$_2$. The model predicts the derivation $e = a \;\text{op}_1\; b$, $f = e \;\text{op}_2\; u$, and the operator identities op$_1$ and op$_2$.

There are $7 \times 7 = 49$ possible operator pairs. One pair is held out entirely from training; the remaining 48 are seen. Training uses balanced marginal sampling (equal per-slot operator frequency) to eliminate unigram shortcuts. Each seen pair contributes ${\sim}3{,}000$ training examples (${\sim}144{,}000$ total) and 500 test examples.

The architecture is an encoder-decoder transformer: 2 encoder layers, 2 decoder layers, 4 attention heads, embedding dimension 64, FFN dimension 128 (${\sim}180$K parameters). The encoder receives 5 integer tokens; the decoder autoregressively produces the derivation sequence.

Three training conditions are tested:
\begin{itemize}
\item \textbf{Full derivation (familiar symbols):} The decoder outputs intermediate values and operator labels using standard symbols ($\oplus$, $\lor$, etc.).
\item \textbf{Full derivation (opaque symbols):} Same structure but operator tokens are replaced with arbitrary letters (e.g., A, B, C\ldots), testing whether generalisation depends on pre-existing symbol meaning.
\item \textbf{Label-only:} The decoder outputs only operator labels, no intermediate computation values, testing whether multi-step structure matters.
\end{itemize}

All 49 holdout pairs are tested (Figure~\ref{fig:coverage}); compact labels use \texttt{\^{}}=XOR, \texttt{|}=OR, \texttt{\&}=AND, \texttt{r}=NOR, \texttt{d}=NAND, \texttt{L}=LSHIFT, and \texttt{R}=RSHIFT. Four pairs are selected for detailed analysis spanning structural diversity: \texttt{\^{}|} (XOR$\to$OR), \texttt{\&L} (AND$\to$LSHIFT), \texttt{R\^{}} (RSHIFT$\to$XOR), and \texttt{d|} (NAND$\to$OR).

\subsection{Complete results}

\begin{table}[htbp]
\centering
\caption{Test-set holdout accuracy (\%) per seed across all conditions. Each cell reports accuracy on a held-out test set, evaluated at the best checkpoint (selected by peak eval-set performance).}
\label{tab:symbolic_full}
\small
\begin{tabular}{llccc}
\toprule
Holdout & Seed & Full (familiar) & Full (opaque) & Label-only \\
\midrule
\texttt{\^{}|} & 42  & 70.4 & 66.0 & 28.0 \\
               & 123 & 74.0 & 78.2 & 29.6 \\
               & 456 & 75.2 & 74.8 & 29.6 \\
               & \textbf{Mean} & \textbf{73.2 $\pm$ 2.0} & \textbf{73.0 $\pm$ 5.1} & \textbf{29.1 $\pm$ 0.8} \\
\midrule
\texttt{\&L}   & 42  & 50.8 & 67.6 & 13.0 \\
               & 123 & 60.2 & 66.8 & 18.0 \\
               & 456 & 54.8 & 63.2 & 12.6 \\
               & \textbf{Mean} & \textbf{55.3 $\pm$ 3.9} & \textbf{65.9 $\pm$ 1.9} & \textbf{14.5 $\pm$ 2.5} \\
\midrule
\texttt{R\^{}} & 42  & 30.4 & 23.2 & 8.0 \\
               & 123 & 35.0 & 33.8 & 8.8 \\
               & 456 & 19.2 & 16.8 & 1.6 \\
               & \textbf{Mean} & \textbf{28.2 $\pm$ 6.6} & \textbf{24.6 $\pm$ 7.0} & \textbf{6.1 $\pm$ 3.2} \\
\midrule
\texttt{d|}    & 42  & 37.0 & 36.4 & 19.4 \\
               & 123 & 47.6 & 63.8 & 39.8 \\
               & 456 & 78.6 & 68.8 & 40.6 \\
               & \textbf{Mean} & \textbf{54.4 $\pm$ 17.7} & \textbf{56.3 $\pm$ 14.2} & \textbf{33.3 $\pm$ 9.8} \\
\bottomrule
\end{tabular}
\end{table}

\subsection{Baselines}

KNN, MLP, and KRR operate on the raw numeric input $(a,b,c,d,u)$. The Oracle baseline enumerates all operator pairs using the true operator tables; Learned Tables first estimate each operator truth table from the seen training chains, then enumerate pairs using those learned tables.

\begin{table}[htbp]
\centering
\caption{Baseline holdout accuracy (\%) for seven representative holdouts. All 49 pairs were evaluated: KNN and MLP achieve exactly 0\% on every holdout; KRR ranges from 0\% to 18.2\% (mean 4.3\%), highest when XOR is the first operator.}
\label{tab:symbolic_baselines}
\small
\begin{tabular}{lcccccc}
\toprule
Holdout & Transformer & KNN & MLP & KRR & Oracle & Learned Tables \\
\midrule
\texttt{\^{}|}  & 73.2 & 0 & 0 & 13.8 & 100 & 100 \\
\texttt{\^{}r}  & 53.4 & 0 & 0 & 11.8 & 100 & 99.6 \\
\texttt{\&L}    & 55.3 & 0 & 0 & 0.2  & 100 & 99.2 \\
\texttt{d|}     & 54.4 & 0 & 0 & 3.2  & 100 & 99.8 \\
\texttt{R\^{}}  & 28.2 & 0 & 0 & 7.2  & 100 & 99.8 \\
\texttt{|\&}    & 12.3 & 0 & 0 & 0.4  & 100 & 100 \\
\texttt{|r}     & 11.6 & 0 & 0 & 0.2  & 100 & 99.2 \\
\bottomrule
\end{tabular}
\end{table}

\subsection{Mechanistic interpretability}

Mechanistic analysis was performed on the seed-42 full-familiar model for each holdout. All results below are from those four models.

\textbf{Position-specific probes.} Linear probes trained on encoder layer~1 representations decode operator identity from specific positions. Op$_1$ is best decoded from positions 0--1 (inputs $a$, $b$), reaching 49--71\% accuracy across holdouts. Op$_2$ is best decoded among the probed positions from position~3 (input $d$), reaching 83--85\%. These probes are descriptive; the corruption test below provides the causal evidence, including for position~2 ($c$).

\begin{table}[htbp]
\centering
\caption{Encoder layer~1 probe accuracy (\%) for operator identity by position.}
\label{tab:probes}
\small
\begin{tabular}{lcccc}
\toprule
 & \multicolumn{2}{c}{Op$_1$ accuracy} & \multicolumn{2}{c}{Op$_2$ accuracy} \\
Holdout & Pos 0 ($a$) & Pos 1 ($b$) & Pos 3 ($d$) & Pos 4 ($u$) \\
\midrule
\texttt{\^{}|} & 63.9 & 61.7 & 83.8 & 56.5 \\
\texttt{\&L}   & 49.2 & 52.3 & 85.0 & 46.9 \\
\texttt{R\^{}} & 70.8 & 67.6 & 84.7 & 37.3 \\
\texttt{d|}    & 58.4 & 53.3 & 84.8 & 44.9 \\
\bottomrule
\end{tabular}
\end{table}

\textbf{Cross-attention maps.} Decoder cross-attention at layer~0 concentrates on selected encoder positions; we treat this as descriptive and rely on the corruption tests below for causal evidence:

\begin{table}[htbp]
\centering
\caption{Cross-attention weights (layer~0 average across heads) on selected encoder positions.}
\label{tab:crossattn}
\small
\begin{tabular}{lcccc}
\toprule
Predicting & \texttt{\^{}|} & \texttt{\&L} & \texttt{R\^{}} & \texttt{d|} \\
\midrule
$e$: attention to $c$ & 84.9\% & 57.2\% & 84.0\% & 91.2\% \\
$f$: attention to $c$ & 74.0\% & 62.3\% & 68.4\% & 77.6\% \\
$f$: attention to $u$ & 24.8\% & 34.9\% & 27.6\% & 19.6\% \\
\bottomrule
\end{tabular}
\end{table}

\textbf{Input corruption (causal test).} Replacing encoder inputs at positions $a$, $b$ changes op$_1$ predictions 62--74\% of the time; replacing $c$, $d$ changes op$_2$ 69--79\%. Replacing $u$ (used only as the fresh operand in the generated second sub-computation) changes op$_1$ $\leq$3\% and op$_2$ $<$2\%, a clean negative control confirming the model has learned causal structure.

\begin{table}[htbp]
\centering
\caption{Input corruption: fraction of predictions changed when each encoder position is replaced with a random value.}
\label{tab:corruption}
\small
\begin{tabular}{lcccc}
\toprule
Corrupted position & \texttt{\^{}|} & \texttt{\&L} & \texttt{R\^{}} & \texttt{d|} \\
\midrule
$a$ $\to$ op$_1$ changed  & 73.0\% & 68.0\% & 62.0\% & 73.7\% \\
$b$ $\to$ op$_1$ changed  & 73.7\% & 73.3\% & 67.0\% & 73.7\% \\
$c$ $\to$ op$_2$ changed  & 70.7\% & 69.7\% & 69.3\% & 70.0\% \\
$d$ $\to$ op$_2$ changed  & 77.0\% & 78.7\% & 75.3\% & 78.3\% \\
$u$ $\to$ op$_1$ changed  & 3.0\%  & 1.0\%  & 1.3\%  & 1.0\%  \\
$u$ $\to$ op$_2$ changed  & 1.7\%  & 0.7\%  & 0.7\%  & 1.0\%  \\
\bottomrule
\end{tabular}
\end{table}

\textbf{Logit lens.} The decoder shows a sharp two-stage computation matching the CA pattern. We project the decoder-input positions that predict the first derivation operator tokens op$_1$ and op$_2$.

\begin{table}[htbp]
\centering
\caption{Logit lens: accuracy (\%) when projecting intermediate decoder representations to output vocabulary at the positions predicting op$_1$ and op$_2$.}
\label{tab:logitlens}
\small
\begin{tabular}{llcccc}
\toprule
Layer & Target & \texttt{\^{}|} & \texttt{\&L} & \texttt{R\^{}} & \texttt{d|} \\
\midrule
Embedding   & op$_2$ & 6.6  & 4.6  & 2.4  & 10.4 \\
After dec-0 & op$_2$ & 97.5 & 87.2 & 82.0 & 89.6 \\
After dec-1 & op$_2$ & 100  & 100  & 99.9 & 100  \\
\midrule
Embedding   & op$_1$ & 0.0  & 3.9  & 1.7  & 0.6 \\
After dec-0 & op$_1$ & 61.8 & 47.9 & 54.3 & 60.1 \\
After dec-1 & op$_1$ & 93.1 & 91.6 & 94.9 & 93.6 \\
\bottomrule
\end{tabular}
\end{table}

\textbf{Head-level ablation.} For the \texttt{\^{}|} holdout, ablating individual attention heads reveals generalisation-specific effects. Ablating \texttt{dec\_cross\_0\_head\_0} drops holdout accuracy from 69\% to 47\% (22pp) while seen accuracy drops only 1.7pp. This head is disproportionately important for generalisation, consistent with the CA finding that specific circuit components are essential for computing the held-out function.

\subsection{Shortcut gradient}

Seen-chain accuracy is 88--94\% for every holdout pair; the model learns the visible task equally well in all cases. The variation is entirely in holdout accuracy, reflecting how easily each held-out composition can be shortcutted. Four case-study holdout pairs illustrate: \texttt{\^{}|} (73.2\%) $>$ \texttt{\&L} (55.3\%) $>$ \texttt{d|}  (54.4\%) $>$ \texttt{R\^{}} (28.2\%). \texttt{\^{}|} (XOR-then-OR) involves two logic operators that share dense algebraic constraints with many seen chains, providing strong indirect signal. \texttt{R\^{}} (RSHIFT-then-XOR) is hardest of the four because RSHIFT is a lossy, many-to-one mapping: right-shifting discards the least significant bit, so multiple inputs map to the same output. The model finds approximate representations that satisfy seen-chain constraints without encoding the precise holdout operator---a shortcut that is sufficient for the visible loss but insufficient for the held-out pair. This is a prediction of the constraint-density account: lossy operators reduce indirect signal and permit shortcut solutions that satisfy visible constraints without uniquely determining the holdout operator. The transformer still exceeds all interpolation baselines at 28.2\% (vs KRR 7.2\%), but the lower accuracy reflects insufficient constraints in the visible data to force true generalisation rather than architectural failure. This mirrors the CA finding that intrinsic rule complexity determines the generalisation frontier.

The full coverage study (Figure~\ref{fig:coverage}) tests all 49 holdout pairs with baselines evaluated on every pair. The transformer exceeds interpolation baselines on all 49 (mean 41.8\% vs KRR mean 4.3\%; KNN and MLP score 0\% everywhere). Per-holdout means range from 9.2\% to 73.2\%, with the hardest pairs involving closely related operators (e.g.\ NAND--NOR, OR--AND) whose similar truth tables reduce the distinguishing signal available from seen chains.

\begin{figure}[htbp]
  \centering
  \includegraphics[height=0.9419\textheight,keepaspectratio]{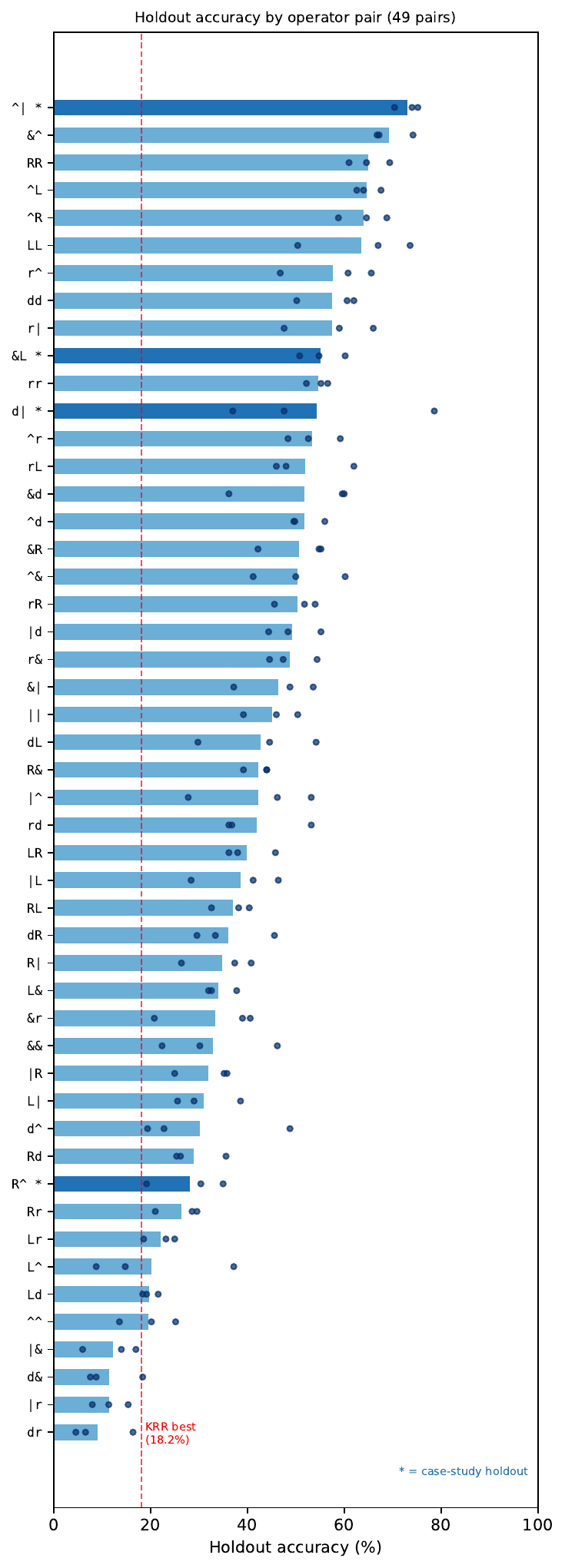}
  \vspace{-0.5em}
  \caption{Holdout accuracy across all 49 operator pairs (3 seeds each, dots show individual seeds). Dashed line: best interpolation baseline (KRR, 18.2\%). Dark bars: four case-study holdouts from Table~\ref{tab:symbolic}.}
  \label{fig:coverage}
\end{figure}

\section{Relationship to Grokking}
\label{app:grokking}

\begin{figure}[htbp]
  \centering
  \includegraphics[width=0.9\textwidth]{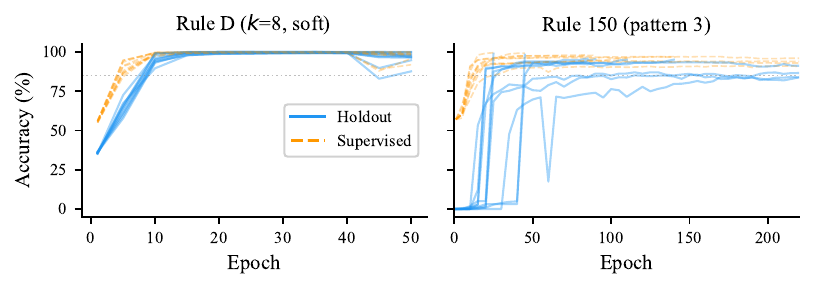}
  \vspace{-0.5em}
  \caption{Training dynamics for Rule~D (left) and Rule~150 pattern~3 (right). Each line is one seed (10 per panel). Holdout accuracy (blue) emerges rapidly once supervised accuracy (orange) exceeds $\sim$85\%. For pure XOR (Rule~150), holdout stays near 0\% until this threshold; Rule~D, which has internal structure permitting partial interpolation, can rise earlier. Rule~150 ($k$=1, radius~1) requires more epochs and fewer seeds converge, consistent with sparser constraints producing weaker indirect signal.}
  \label{fig:phase}
\end{figure}

Our phase transition resembles grokking \citep{power2022grokking}, but differs in a key respect: the model never memorises the hidden outputs because they are never provided. The delay reflects a constraint propagation threshold: until the model's visible-pattern predictions are accurate enough, wrong hidden-pattern predictions do not produce detectable errors at downstream visible positions. The bimodal outcome (each seed either fully succeeds or fully fails) is consistent with the lottery ticket hypothesis \citep{frankle2018lottery} and initialization-dependent phase transitions \citep{zhang2024initialization}.

\end{document}